\documentclass{article}

\usepackage{microtype}
\usepackage{graphicx}
\usepackage{booktabs} 
\usepackage{capt-of}
\usepackage{tabularx}
\usepackage{subcaption}

\usepackage{hyperref}

\usepackage[accepted]{icml2024}

\usepackage{amsmath}
\usepackage{amssymb}
\usepackage{mathtools}
\usepackage{amsthm}
\usepackage{twoopt}
\usepackage{arydshln}
\usepackage{float}
\newcommand{\R}{{\mathbb R}}

\newcommandtwoopt{\st}[2][t][]{{s_{#1}^{#2}}}
\newcommandtwoopt{\ac}[2][t][]{{a_{#1}^{#2}}}

\usepackage[capitalize,noabbrev]{cleveref}

\theoremstyle{plain}
\newtheorem{theorem}{Theorem}[section]

\newtheorem{lemma}[theorem]{Lemma}

\theoremstyle{definition}
\newtheorem{definition}[theorem]{Definition}

\theoremstyle{remark}
\newtheorem{remark}[theorem]{Remark}

\usepackage[textsize=tiny]{todonotes}

\usepackage{mathtools}

\newcommand\blfootnote[1]{%
  \begingroup
  \renewcommand\thefootnote{}\footnote{#1}%
  \addtocounter{footnote}{-1}%
  \endgroup
}

\definecolor{darkgreen}{rgb}{0.0, 0.8, 0.0}

\begin{document}

\twocolumn[
\icmltitle{Variance-reduced Zeroth-Order Methods for Fine-Tuning Language Models
}



\begin{icmlauthorlist}
\icmlauthor{Tanmay Gautam\textsuperscript{*}}{ucb}
\icmlauthor{Youngsuk Park\textsuperscript{\dag}}{air}
\icmlauthor{Hao Zhou}{ailabs}
\icmlauthor{Parameswaran Raman}{air}
\icmlauthor{Wooseok Ha}{ailabs}
\end{icmlauthorlist}

\icmlaffiliation{ailabs}{Amazon AI Labs, Santa Clara, USA}
\icmlaffiliation{air}{Amazon AI Research \& Education, Santa Clara, USA}
\icmlaffiliation{ucb}{University of California, Berkeley, USA}

\icmlcorrespondingauthor{Youngsuk Park\textsuperscript{\dag}}{pyoungsu@amazon.com}

\icmlkeywords{Zeroth-Order Optimization, Language Models, Fine-tuning, Variance Reduction}

\vskip 0.3in
]



\printAffiliationsAndNotice{}  

\begin{abstract}
Fine-tuning language models (LMs) has demonstrated success in a wide array of downstream tasks. However, as LMs are scaled up, the memory requirements for backpropagation become prohibitively high. Zeroth-order (ZO) optimization methods can leverage memory-efficient forward passes to estimate gradients. Recently, MeZO, an adaptation of ZO-SGD, has been shown to consistently outperform zero-shot and in-context learning when combined with suitable task prompts. In this work, we couple ZO methods with variance reduction techniques to enhance stability and convergence for inference-based LM fine-tuning. We introduce Memory-Efficient Zeroth-Order Stochastic Variance-Reduced Gradient (MeZO-SVRG) and demonstrate its efficacy across multiple LM fine-tuning tasks, eliminating the reliance on task-specific prompts. Evaluated across a range of both masked and autoregressive LMs (up to 7B parameters) on benchmark downstream tasks, MeZO-SVRG outperforms MeZO with up to 20\% increase in test accuracies in both full- and partial-parameter fine-tuning settings. MeZO-SVRG benefits from reduced computation time, often surpassing MeZO's peak test accuracy with a $2\times$ reduction in GPU-hours. MeZO-SVRG substantially decreases the memory requirement (by at least $2\times$ for autoregressive models), achieving greater memory savings as both the batch size and context lengths increase, in comparison to first-order methods.

\end{abstract}

\section{Introduction}\label{section:intro}
\begin{figure*}[htbp]
    \centering
    \begin{minipage}{0.65\columnwidth}
    \centering
        \includegraphics[width=\textwidth]{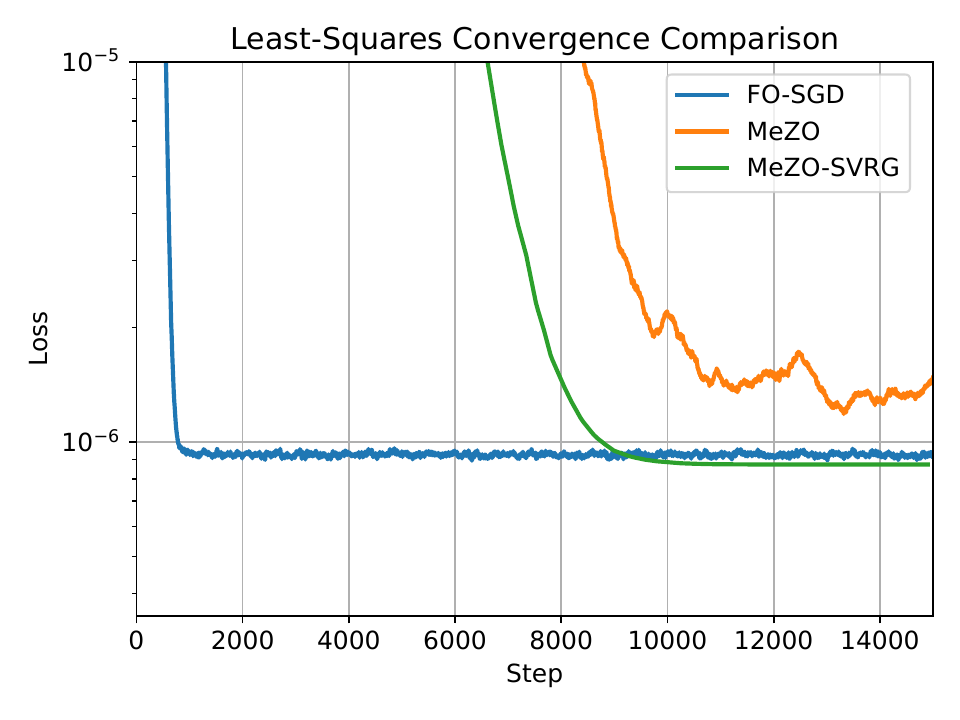}
        \subcaption{}
        \label{fig:leastsquares}
    \end{minipage}
    \begin{minipage}{0.65\columnwidth}
    \centering
        \includegraphics[width=\linewidth]{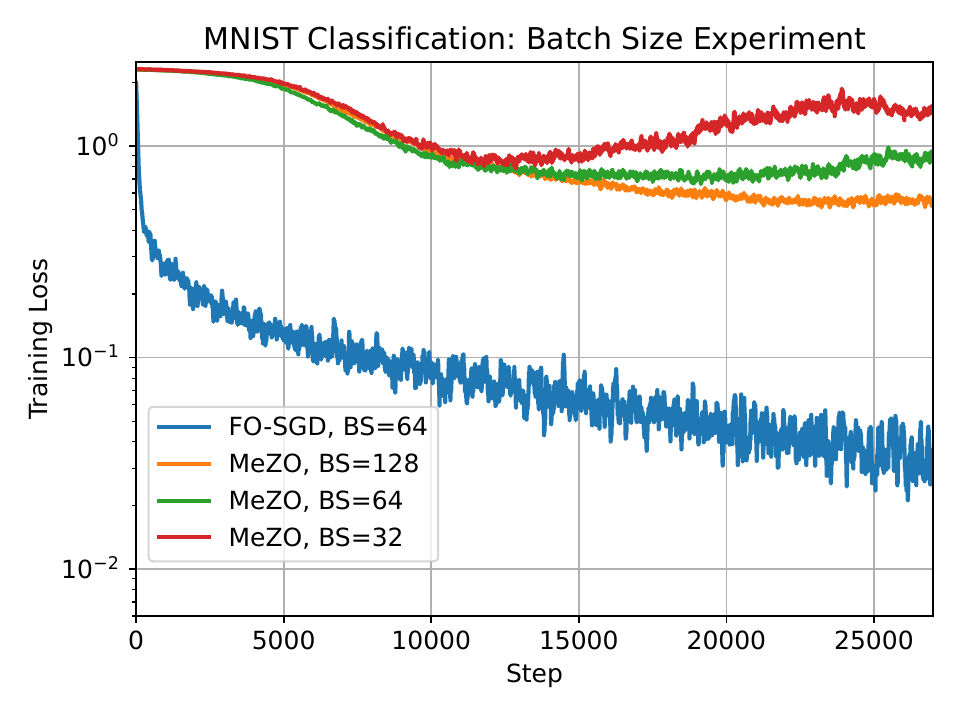}
        \subcaption{}
        \label{fig:mnistbatch}
    \end{minipage}
    \begin{minipage}{0.65\columnwidth}
    \centering
        \includegraphics[width=\linewidth]{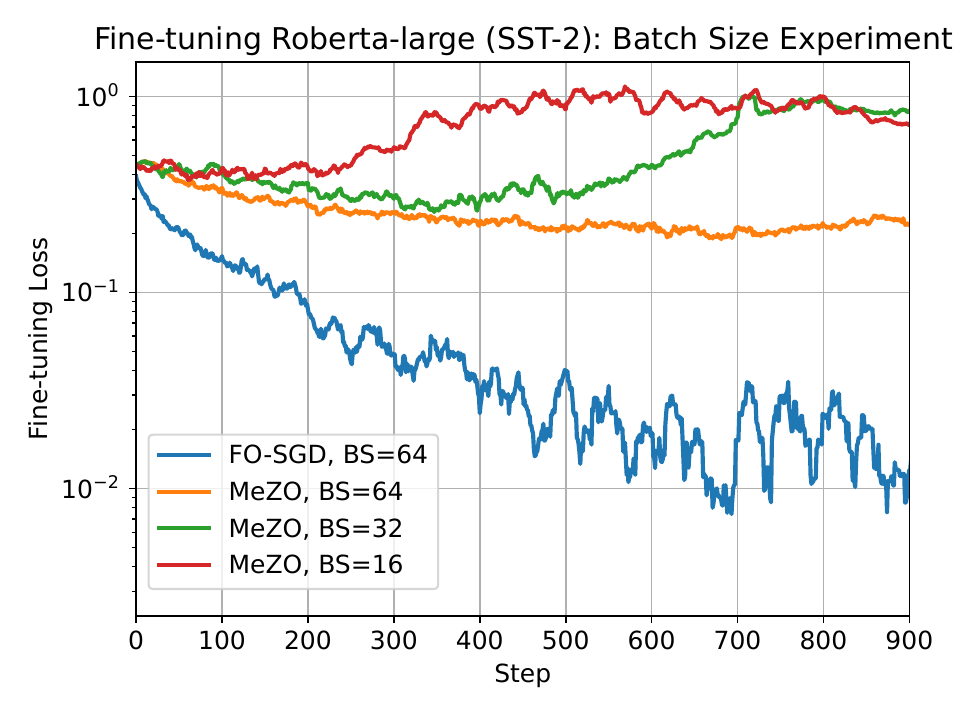}
        \subcaption{}
        \label{fig:robertabatch}
    \end{minipage}
    \caption{(a) Shows that MeZO \citep{malladi2023mezo} is unable to attain the optimal value when solving least-squares (LS) problems unlike our proposed MeZO-SVRG. In (b) and (c), MeZO  is used for MNIST \citep{lecun1998gradient} classification and fine-tuning RoBERTa-large on SST-2 \cite{socher-etal-2013-recursivesst2}, respectively, with varying batch sizes. These illustrate MeZO's instability w.r.t. smaller batch sizes. 
    }
\end{figure*}

In recent years, language models (LMs) have exhibited exceptional performance in a vast array of domains within natural language processing (NLP) \citep{solaiman2019release, openai2023gpt4, touvron2023llama}.
This development has generated immense excitement within the research community and has propelled the advancement of the aforementioned models to the forefront of deep learning research.

\blfootnote{\textsuperscript{*}Work conducted during an internship at Amazon. }


Fine-tuning LMs has been the dominant strategy for adapting pre-trained models to specialized downstream tasks \citep{gururangan-etal-2020-dont}.
Fine-tuning often relies on first-order methods, such as stochastic gradient descent (SGD) \citep{sgd} or Adam \citep{Kingma2015AdamAM}. However, as LMs are scaled up, backpropagation \citep{Rumelhart1986LearningRB} becomes prohibitive in terms of memory requirements. 
More concretely, \citet{malladi2023mezo} show that fine-tuning an OPT-13B model with full-parameter or parameter efficient fine-tuning (PEFT) using Adam requires 12$\times$ and 6$\times$ more memory than inference, respectively. This is due to the need to cache activations during the forward pass as well as gradients and optimizer states during the backward pass.
This has given rise to memory-efficient inference-based adaptation methods, including in-context learning (ICL) and zeroth-order (ZO) optimization. 

While ZO methods have been studied for decades \cite{spsa, zosgd}, it is only recently that these have been applied to fine-tune LMs \citep{malladi2023mezo}. In \citet{malladi2023mezo}, authors propose the Memory-Efficient Zeroth-Order Optimizer (MeZO) and demonstrate its superior performance against ICL with a memory footprint equivalent to that of inference. By virtue of estimating gradients through loss computations, ZO methods are compatible with settings where gradients are non-accessible or infeasible to compute, e.g. when considering non-differentiable objectives or black-box access of LMs.

However, ZO methods still face challenges in large-scale settings. According to \citet{malladi2023mezo}, MeZO requires a high number of iterations to achieve a good fine-tuning performance and works only in settings where the optimization trajectory is sufficiently well-behaved, i.e. when fine-tuning is coupled with appropriately crafted task prompts.
As such, we revisit ZO optimization under the standard (non-prompted) fine-tuning setting. Through empirical studies, we probed further and identified that the method also contends with i) instability for smaller batch sizes, and ii) a notable convergence gap to first-order (FO) fine-tuning methods in non-prompted settings (see Figures \ref{fig:leastsquares}, \ref{fig:mnistbatch}, \ref{fig:robertabatch}).


In this work, we demonstrate that variance-reduction enhances the stability and convergence properties of ZO methods in the large-scale LM fine-tuning setting. Based on our observation that ZO methods benefit from improved stability with larger batch sizes, we propose the Memory Efficient Zeroth-Order Stochastic Variance-Reduced Gradient (MeZO-SVRG) method: a ZO algorithm that combines fullbatch and minibatch information to yield asymptotically unbiased, low-variance gradient estimators. Our specific contributions are enumerated below.

\begin{enumerate}
    \item We perform empirical studies across a range of problem scales to investigate the potential limitations of MeZO. We identified its susceptibility to unstable behavior for smaller batch sizes and convergence issues in spurious optimization landscapes as improvement avenues.
    \item We propose MeZO-SVRG: an efficient variant of the ZO-SVRG method that leverages gradient estimators computed with single perturbation vectors to exploit data parallelism for speed and uses in-place operations to achieve a minimal memory footprint. 
    \item We fine-tune masked and autoregressive LMs (model scales up to 7B) on GLUE  \citep{wang-etal-2018-glue} and SuperGLUE \citep{superglue} tasks. MeZO-SVRG achieves consistent performance improvements with up to 20\% increase in test accuracies over MeZO across all models and tasks. MeZO-SVRG achieves superior performance to MeZO in both full- and partial-parameter fine-tuning, in both full (FP32) and half (BF16) precision and under standard non-prompt settings.
    \item MeZO-SVRG stands out by consistently surpassing MeZO's test accuracy in only half as many GPU-hours.
    \item We show that MeZO-SVRG significantly reduces the required memory footprint compared to first-order methods, i.e. by at least $2\times$ for considered autoregressive models. Furthermore, our experiments highlight that MeZO-SVRG's memory savings progressively improve compared to SGD with larger batch sizes.
    \item We establish convergence guarantees for MeZO-SVRG when equipped with gradient estimators that are computed using single perturbation vectors. 
\end{enumerate}

\section{Background}\label{section:background}
\subsection{Zeroth-Order Gradient Estimators} 
Consider solving the unconstrained optimization
\begin{align}\label{eq:empiricalriskmin}
 \min_{\boldsymbol{\theta}\in\R^d} f(\boldsymbol{\theta}):=\frac{1}{n}\sum_{i=1}^n f_i(\boldsymbol{\theta}),    
\end{align}
where $f:\R^d\rightarrow \R$ is a non-convex objective. Note that \eqref{eq:empiricalriskmin} is akin to the standard empirical risk minimization framework, where each $f_i$ is the objective evaluated for one of $n$ training samples. For an iterative ZO algorithm, we need to find a means to approximate the gradient. We can define the following stochastic perturbation simultaneous approximation (SPSA) gradient estimator \citep{spsa}:
\begin{align}
    \hat{\nabla} f_i(\boldsymbol{\theta}):=&\frac{f_i(\boldsymbol{\theta} + \mu \textbf{z}_i) - f_i(\boldsymbol{\theta} - \mu \textbf{z}_i)}{2\mu}\textbf{z}_i \textrm{ for } i\in[n],
\end{align}
where $\hat{\nabla}$ denotes a gradient estimator, $\textbf{z}_i\in\R^d$ is a random vector sampled from a standard normal distribution, and $\mu>0$ is a perturbation scalar. The extension $p$-SPSA computes the average of $p$ distinct SPSA estimates. 
Throughout this work, we consider the default setting of $p=1$ as we didn't observe empirical benefits of setting $p>1$. The SPSA gradient estimate is an asymptotically unbiased estimator of the true gradient as $\mu\to 0$ when each component in $\textbf{z}_i$ is mutually independent and zero-mean \citep{spsa}.

Now suppose we have a minibatch $\mathcal{I}\subset [n]$ of size $b$. 
This allows us to define the following:
\begin{align}\label{eq:minibatchspsa1}
    \hat{\nabla} f_{\mathcal{I}}(\boldsymbol{\theta}):=\frac{1}{b}\sum_{i\in \mathcal{I}} \hat{\nabla} f_i(\boldsymbol{\theta}),
\end{align}
and by extension,
\begin{align}\label{eq:fullbatchspsa1}
    \hat{\nabla} f(\boldsymbol{\theta}):= \hat{\nabla} f_{[n]}(\boldsymbol{\theta}).
\end{align}    
Observe that the gradient estimator in \eqref{eq:minibatchspsa1}  requires $2b$ function queries and sampling $b$ random vectors. In practice, there are two strategies to compute estimators \eqref{eq:minibatchspsa1} and \eqref{eq:fullbatchspsa1}: accumulate the minibatch estimator in-place by sequentially computing each samplewise estimator, or parallelize the operation by computing the samplewise estimators simultaneously. The trade-off between the two strategies is that the former has a minimal memory footprint (scales with dimension of problem) but takes longer, while the latter effectively parallelizes the operation but has to store $b$ vectors.

Thus, we define another set of ZO gradient estimators that accommodate data parallelism: we perturb each samplewise SPSA estimator in the same direction $\textbf{z}\in\R^d$. For minibatch $\mathcal{I}\subset [n]$ of size $b$ we can construct
\begin{align}\label{eq:minibatchspsa2}
    \bar{\nabla} f_{\mathcal{I}}(\boldsymbol{\theta}):=\frac{\frac{1}{b}\sum_{i\in\mathcal{I}}[f_i(\boldsymbol{\theta} + \mu \textbf{z}) - f_i(\boldsymbol{\theta} - \mu \textbf{z})]}{2\mu}\textbf{z},
\end{align}
and
\begin{align}\label{eq:fullbatchspsa2}
    \bar{\nabla} f(\boldsymbol{\theta}):= \bar{\nabla} f_{[n]}(\boldsymbol{\theta}).
\end{align}

From an implementation standpoint, estimators \eqref{eq:minibatchspsa2} and \eqref{eq:fullbatchspsa2} can exploit data parallelism across the batch $\mathcal{I}$ and benefit from a minimal required memory footprint.



\subsection{Memory-efficient ZO-SGD (MeZO)}
In \citet{malladi2023mezo}, the authors propose a memory-efficient ZO-SGD optimizer (MeZO) to fine-tune LMs. MeZO is a ZO-SGD algorithm that estimates gradients based on the two-point SPSA estimator introduced in~\eqref{eq:minibatchspsa2}. 

\begin{definition}\label{def:zosgd}
    \textit{(ZO-SGD)} Consider solving optimization \eqref{eq:empiricalriskmin}. ZO-SGD is an iterative ZO optimizer characterized with update rule
    $$\boldsymbol{\theta}^{(t+1)}:= \boldsymbol{\theta}^{(t)} - \eta \bar{\nabla}f_{\mathcal{I}}(\boldsymbol{\theta}^{(t)}),$$ for learning rate $\eta>0$, and SPSA estimator $\bar{\nabla}f_{\mathcal{I}}(\boldsymbol{\theta}^{(t)})$ over minibatch $\mathcal{I}\in[n]$.
\end{definition}

Implementing a vanilla ZO-SGD algorithm requires twice the memory footprint of inference due to the need to store the perturbation vector $\textbf{z}\in\R^d$. In \citet{malladi2023mezo}, an in-place implementation of the algorithm is proposed, where the requirement of storing a full set of perturbation scalars is mitigated by merely storing a single random seed and regenerating the perturbation vector when required. This brings the memory cost of MeZO down to that of inference (see Appendix \ref{appendix:spsaimplementation} for more details on the implementation).

\subsection{ZO-SVRG}\label{ssection:zosvrg}

The Zeroth-Order Stochastic Variance Reduced Gradient (ZO-SVRG) \citep{zosvrg} method periodically combines a fullbatch gradient estimator with the minibatch estimator to mitigate the stochasticity of the latter. This variance reduction helps achieve a faster convergence rate compared to ZO-SGD \citep{zosvrg}. While the full algorithm is presented in Appendix \ref{appendix:zosvrg}, the update rule is:
\begin{align}\label{eq:zosvrg}
    \boldsymbol{\theta}^{(t+1)} \leftarrow \boldsymbol{\theta}^{(t)} - \eta [\hat{\nabla}f_{\mathcal{I}_t}(\boldsymbol{\theta}^{(t)}) - \hat{\nabla}f_{\mathcal{I}_t}(\bar{\boldsymbol{\theta}}) + \hat{\nabla}f(\bar{\boldsymbol{\theta}})]
\end{align}

where $\eta>0$ is the learning rate, $\mathcal{I}_t$ is a minibatch sampled at iteration $t$, $\boldsymbol{\theta}^{(t)}$ is the parameter state at iteration $t$, and $\bar{\boldsymbol{\theta}}$ is the last parameter state at which the fullbatch gradient estimator was computed. Throughout this work, we let $q\in\mathbb{N}$ denote the regularity of fullbatch SPSA computations, i.e. every $q$ steps the fullbatch SPSA estimator is computed.

\section{Our proposed method: MeZO-SVRG}\label{section:method}

In this section, we describe the proposed MeZO-SVRG method. We first motivate our method by discussing the observed limitations of MeZO and outline practical implementation concerns when using ZO-SVRG \citep{zosvrg} to mitigate these. 
We then introduce MeZO-SVRG as a variant of ZO-SVRG that minimizes memory usage with in-place operations and accommodates data parallelism in its gradient estimators.

\subsection{MeZO Limitations}\label{sssection:mezolimitations}

In \citet{malladi2023mezo}, authors mention that MeZO requires a suitable task prompt to perform well; under this setting the optimization trajectory is more well-behaved. This suggests that the applicability of MeZO is restricted to settings where the optimization landscape is sufficiently well-behaved and cannot be extended to more complex tasks such as pre-training. 
Moreover, the careful design of prompts for real-world fine-tuning tasks also demands additional effort and may not always be practical. This motivates developing a method that delivers robust performance independently of any reliance on input prompts.

While MeZO has demonstrated promise in fine-tuning settings, our empirical findings suggest that it still faces the following challenges: i) it is susceptible to instability when using smaller batch sizes, and ii) a considerable performance gap with respect to first-order (FO) fine-tuning exists in the non-prompted setting. We illustrate these issues in Figures \ref{fig:leastsquares}, \ref{fig:mnistbatch} and \ref{fig:robertabatch}.
The details of the experiments are provided in Appendix \ref{appendix:limitsmezo}. 
These observations motivate using variance-reduction techniques that leverage larger batch information to improve stability and convergence of ZO methods in the large-scale problem settings.

\subsection{ZO-SVRG Implementation Concerns}
\textbf{Memory Footprint.} 
Recalling $\boldsymbol{\theta}\in\R^d$, the ZO-SVRG method has a minimum memory requirement of storing $d$ values. A naive implementation of ZO-SVRG presented in Algorithm \ref{alg:method1} (see Appendix~\ref{appendix:zosvrg}) would require an additional $2d$ of memory space for storing the fullbatch gradient estimator and parameter state $\bar{\boldsymbol{\theta}}$ used in \eqref{eq:zosvrg}. Moreover, computing and storing $\hat{\nabla}f_{\mathcal{I}_t}(\boldsymbol{\theta}^{(t)})$ and $\hat{\nabla}f_{\mathcal{I}_t}(\bar{\boldsymbol{\theta}})$ also accrues an additional $d$ values of memory each. Thus, a naive implementation of Algorithm \ref{alg:method1} would require a minimum memory budget equivalent to $5\times$ the memory budget of inference, which is prohibitive for sufficiently large $d$. 

\textbf{Iteration Speed Concerns.} The original ZO-SVRG method is proposed with the inefficient gradient estimators introduced in \eqref{eq:minibatchspsa1} and \eqref{eq:fullbatchspsa1}. In both, SPSA estimators are computed for individual samples and averaged over the batch. Consider computing \eqref{eq:minibatchspsa1} with batch size $b$. If we want to fully parallelize operations, we require computing and storing $b$ many $\hat{\nabla} f_i(u)$ estimators. However, this increases the memory footprint. To save on memory usage, in-place operations can be used. However, this has the effect of drastically reducing the computation speed as we need to sequentially compute each of the $b$ estimators in \eqref{eq:minibatchspsa1}. 


\subsection{MeZO-SVRG}\label{ssection:MeZO-SVRG}
We propose MeZO-SVRG: a variant of ZO-SVRG that improves iteration speed by using estimators \eqref{eq:minibatchspsa2}, \eqref{eq:fullbatchspsa2} and reduces the memory footprint with in-place operations. The method is summarized in Algorithm \ref{alg:method2}.

\textbf{Efficient Gradient Estimation.} We utilize the efficient gradient estimators introduced in \eqref{eq:minibatchspsa2} and \eqref{eq:fullbatchspsa2} that perturb the entire batch in a single direction. These estimators accommodate data parallelism offered by modern ML frameworks. Furthermore, we can utilize the ``resampling trick'' introduced in \citet{malladi2023mezo} to reduce the memory footprint when computing each of \eqref{eq:minibatchspsa2} and \eqref{eq:fullbatchspsa2}; each estimator requires a memory footprint equivalent to the problem dimension $d$ (see Appendix \ref{appendix:spsaimplementation} for the memory-efficient SPSA computation procedure). Thus, using estimators \eqref{eq:minibatchspsa2} and \eqref{eq:fullbatchspsa2}, eliminates the memory/speed trade-off plaguing the ZO-SVRG implementation and get the best of both worlds.

\begin{algorithm}[h]
   \caption{Memory-Efficient ZO-SVRG (MeZO-SVRG)}
   \label{alg:method2}
   \footnotesize
\begin{algorithmic}
    \STATE {\bfseries Input:} Total iterations $T$, learning rates $\eta_1, \eta_2>0$, minibatch size $b$, parameters $\boldsymbol{\theta}_0$, iterations between full-batch gradient $q\in\mathbb{N}$ 
   \STATE {\bfseries begin method}
    \FOR{$t = 0, \dots, T$}
   \IF {$t \mod q = 0$}
   \STATE 1. $\textbf{g} \leftarrow \bar{\nabla}f(\boldsymbol{\theta}^{(t)})$
   \STATE 2. $\bar{\boldsymbol{\theta}} \leftarrow \boldsymbol{\theta}^{(t)}$ 
   \STATE 3. update: $\boldsymbol{\theta}^{(t+1)} \leftarrow \boldsymbol{\theta}^{(t)} - \eta_1 \textbf{g} \quad$\texttt{\#in-place}
   \ELSE
   \STATE 4. Choose mini-batch $\mathcal{I}_t$ of size b
   \STATE 5. $\boldsymbol{\theta}^{(t)} \leftarrow \boldsymbol{\theta}^{(t)} - \eta_2\bar{\nabla}f_{\mathcal{I}_t}(\boldsymbol{\theta}^{(t)})\quad $ \texttt{\#in-place}
    \STATE 6. $\boldsymbol{\theta}^{(t)} \leftarrow \boldsymbol{\theta}^{(t)} + \eta_2\bar{\nabla}f_{\mathcal{I}_t}(\bar{\boldsymbol{\theta}})\quad $ \texttt{\#in-place}
   \STATE 7. update: $\boldsymbol{\theta}^{(t+1)} \leftarrow \boldsymbol{\theta}^{(t)} - \eta_2 \textbf{g}\quad$ \texttt{\#in-place}
   \ENDIF{}
   \ENDFOR{}
   \STATE {\bfseries end} 
\end{algorithmic}
\end{algorithm}

\begin{figure*}[!ht]
    \centering
    \begin{minipage}{\columnwidth}
    \centering
        \includegraphics[width=0.7\linewidth]{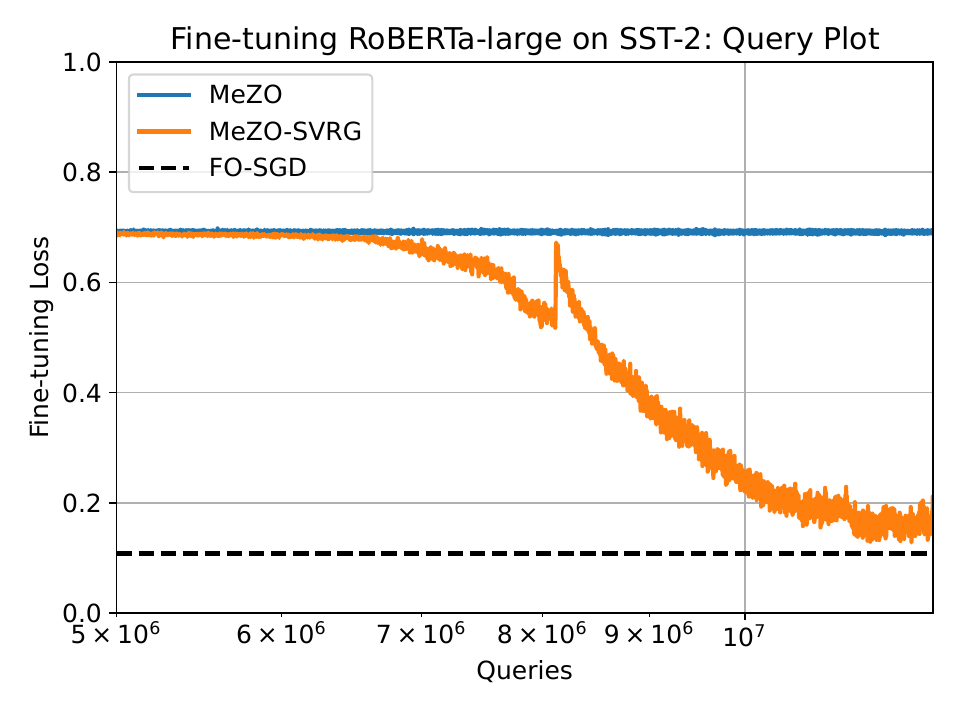}
        \subcaption{}
        \label{fig:convergencerobertaquery}
    \end{minipage}
    \hfill
    \begin{minipage}{\columnwidth}
    \centering
        \includegraphics[width=0.7\linewidth]{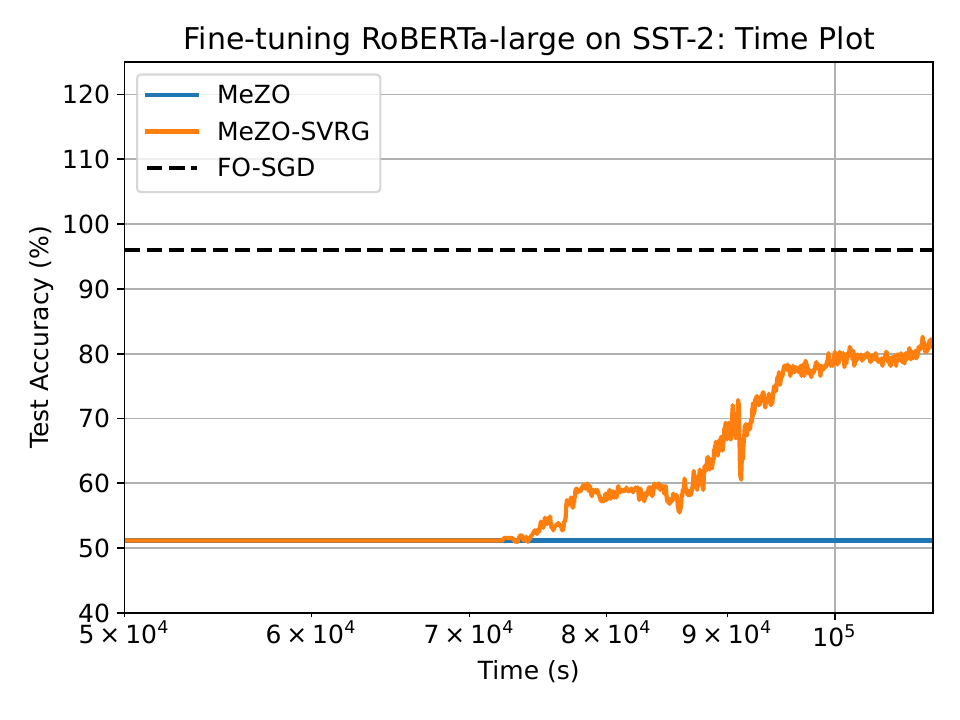}
        \subcaption{}
        \label{fig:accuracyrobertatime}
    \end{minipage}
    \caption{Performance of MeZO-SVRG, MeZO and FO-SGD when fine-tuning RoBERTa-large on the SST-2 \citep{socher-etal-2013-recursivesst2} dataset. The dashed line serves as a reference to the training loss/test accuracy achieved by FO-SGD. (a) MeZO-SVRG is able to significantly reduce the convergence gap to FO-SGD compared to MeZO. (b) MeZO-SVRG attains a considerably better test accuracy than MeZO.}
\end{figure*}

\textbf{In-place Operations for Memory Efficiency.} MeZO-SVRG leverages in-place operations to minimize memory allocation for new variable definitions. Memory space is required for the current state of the $d$ parameters, a copy of the parameter state after each fullbatch SPSA computation as well as the fullbatch SPSA estimator itself. This requires a minimum memory requirement of storing $3d$ values. 
The minibatch updates can then be computed in-place in Lines 5, 6, and 7; thus, MeZO-SVRG achieves a reduced minimum memory footprint to $3\times$ that of inference.

\begin{remark}
As MeZO-SVRG queries the loss function with an inference pass through a network, it minimizes the storage of activations and intermediate variables. The memory footprint of MeZO-SVRG thus mainly stems from retaining copies of the fullbatch gradient estimator and parameters. Therefore, this method scales well with increasing batch sizes. Table \ref{tab:memory} shows that for increasing batch sizes of up to 64, MeZO-SVRG yields more than 70\% memory savings compared to first-order SGD (FO-SGD) on the RoBERTa-large \citep{robertalarge} model. Similarly, MeZO-SVRG improves significantly on memory usage compared to FO-SGD for large context lengths and a fixed batch size (consistently $2\times$ smaller footprint, see Figure \ref{fig:memoryplot}). 
\end{remark}
\vspace{0.2em}
\begin{remark}
By storing $\bar{\boldsymbol{\theta}}$ in Line 2 (Algorithm \ref{alg:method2}), we can keep recomputing the fullbatch estimator on demand without storing $\textbf{g}$. This would lower the memory footprint of MeZO-SVRG to $2\times$ that of inference. However, as computing the fullbatch estimator can slow down the iteration speed, throughout this work we store it.
\end{remark}
\vspace{0.2em}
\begin{remark}
The memory analysis above does not account for any constant implementation overhead or intermediate activation storage during a forward pass of a network. Thus, in practice the memory usage ratio between MeZO and MeZO-SVRG is smaller (see Table \ref{tab:memory}).
\end{remark}
\vspace{0.2em}
\begin{remark}
In practice, fine-tuning datasets can be large enough that computing fullbatch SPSA estimators is infeasible (e.g. more than $10^5$ training examples). MeZO-SVRG can be adapted so that the fullbatch estimator is approximated with a large batch estimator (e.g. with 512 or 1024 samples). In this case, the updates blend minibatch and large batch (as opposed to fullbatch) information.
\end{remark}

\textbf{Additional Learning Rate.} In Algorithm \ref{alg:method2}, we also include two independent learning rates $\eta_{1}$ and $\eta_{2}$ for the fullbatch and minibatch updates as shown in Lines 3 and Lines 5-7, respectively, of Algorithm \ref{alg:method2}. This design choice is based on our empirical observation that fullbatch updates are more accommodating of larger learning rates than minibatch steps. In our experiments we find that setting $\eta_1 > \eta_2$ improves convergence speed (see Appendices \ref{appendix:distilberthyperparameter}, \ref{appendix:robertalargehyperparameter}, \ref{appendix:hyperparameterautoregressive}). 
\renewcommand{\arraystretch}{1.1}
\begin{table*}[!h]
\centering{\footnotesize
            \begin{tabular}{lcccc cccc}
            \toprule
             & \multicolumn{4}{c}{\textbf{DistilBert} Full-Precision (FP32)} & \multicolumn{4}{c}{\textbf{RoBERTa-large} Full-Precision (FP32)}  \\
            \cmidrule(lr){2-5} \cmidrule(lr){6-9}
              \textbf{Method}  & MNLI   &  QNLI & SST-2 & CoLA & MNLI  &  QNLI & SST-2  & CoLA \\
            \midrule
            MeZO (Full FT)  & 36 (1.09) & 50 (0.69) & 52 (0.68) & 63 (0.64) & 43 (0.94) & 59 (0.58) & 56 (0.69) & 68 (0.51) \\
            MeZO-SVRG (Full FT)  & \textbf{46 (0.08)} & \textbf{68 (0.23)} & \textbf{72 (0.02)} & \textbf{68 (0.28)} & \textbf{49 (0.81)} & \textbf{80 (0.28)} & \textbf{84 (0.13)} & \textbf{79 (0.06)} \\
            \hdashline
            FO-SGD (Full FT)  &  59 (0.01) & 78 (0.04) & 88 (0.01) & 70 (0.02) & 85 (0.03) & 89 (0.01) & 96 (0.11) & 85 (0.01) \\
            \midrule
            MeZO (Partial FT)   & 35 (1.09) & 52 (0.69) & 51 (0.70) & 60 (0.64) & 42 (1.07)& 50 (0.69)& 54 (0.68) & 65 (0.59) \\
            MeZO-SVRG (Partial FT)  & \textbf{47 (0.28)} & \textbf{65 (0.29)} & \textbf{74 (0.10)} & \textbf{67 (0.36)} & \textbf{43 (0.82)} & \textbf{67 (0.46)} & \textbf{72 (0.59)} & \textbf{79 (0.35)} \\
            \hdashline
            FO-SGD (Partial FT)  &  48 (0.26)& 59 (0.42) & 85 (0.05) & 66 (0.45) & 52 (0.99) & 72 (0.60) & 89 (0.58) & 84 (0.41)\\
            \bottomrule
            \end{tabular}}
            \caption{Experiments on DistilBert and RoBERTa-large. We show the test accuracies and fine-tuning losses (in parentheses) of MeZO-SVRG and MeZO for both full/partial-parameter FT. We also provide results for FO-SGD as an upper-bound benchmark on performance. MeZO-SVRG consistently outperforms MeZO and significantly closes the gap to FO-SGD. 
            }
            \label{tab:accuraciesmasked}
\end{table*}


\begin{table*}[!h]
\centering{\footnotesize
            \begin{tabular}{lcccccccc}
            \toprule
             & \multicolumn{3}{c}{\textbf{GPT2} Full-Precision (FP32)} & \multicolumn{3}{c}{\textbf{OPT-2.7B} Full-Precision (FP32)} & \multicolumn{2}{c}{\textbf{OPT-6.7B} Half-Precision (BF16)}\\
            \cmidrule(lr){2-4} \cmidrule(lr){5-7}
            \cmidrule(lr){8-9}
              \textbf{Method}  & MNLI   & SST-2 & CoLA & MNLI   & SST-2 & CoLA & SST-2 & BoolQ\\
            \midrule
            MeZO   &  41 (0.65) & 59 (0.32) & 61 (0.35) & 42 (1.09) & 61 (0.65) & 62 (0.58) & 74 (0.53) & 65 (0.63)\\
            MeZO-SVRG   & \textbf{53 (0.41)} &  \textbf{65 (0.20)}  & \textbf{69 (0.25)} & \textbf{52 (0.81)} &  \textbf{65 (0.55)} & \textbf{67 (0.53)} &
            \textbf{77 (0.52)} &
            \textbf{69 (0.57)}\\
            \hdashline
            FO-SGD   &  69 (0.59) & 72 (0.23) & 78 (0.38) & 78 (0.33)  & 98 (0.02) & 94 (0.17) & 91 (0.10) & 84 (0.29)\\
            \bottomrule
            \end{tabular}}
            \caption{Experiments on AR models. We show the test accuracies and fine-tuning losses (in parentheses) of MeZO-SVRG and MeZO for full-parameter FT. For reference we also provide results for FO-SGD as an upper-bound benchmark on performance. MeZO-SVRG consistently outperforms MeZO and approaches FO-SGD performance. 
            }
            \label{tab:accuraciesautoregressive}
\end{table*}

\textbf{Storage Efficiency of MeZO-SVRG.} Parameter-efficient fine-tuning (PEFT) reduces the size of fine-tuned model checkpoints by optimizing only a small subset of parameters, e.g. 
LoRA \citep{hu2022lora} and prefix-tuning \citep{li-liang-2021-prefix}. Both MeZO and MeZO-SVRG have the benefit of being able to recover an entire fine-tuning trajectory by storing a single seed and the difference of loss scalars in \eqref{eq:minibatchspsa2} at each step. The stored seed can regenerate step-wise seeds to recover the perturbation vectors $\textbf{z}$ used for each SPSA computation. Together with the stored difference in loss values, we can recover the exact gradient estimators used in the fine-tuning process without needing to perform any forward passes. This allows recovering any model checkpoint along the fine-tuning trajectory. As we store only the initial random seed and a sequence of difference of loss scalars, we can achieve significant storage efficiency.

\textbf{Compatibility with Non-differentiable Objectives and PEFT.} 
As MeZO-SVRG uses only forward passes and a difference of loss values to estimate the gradient, it is applicable to settings where gradients are inaccessible or infeasible to compute, e.g. when considering non-differentiable objectives such as ranking in RLHF \citep{rlhf} or access to model gradients is restricted. Similar to MeZO, MeZO-SVRG also remains compatible with PEFT (e.g. LoRA \citep{hu2022lora}, prefix-tuning \citep{li-liang-2021-prefix}). 

\section{Experiments}\label{section:exp}
\begin{table*}
    \centering{\footnotesize
        \begin{tabular}{l ccccc cc}
            \toprule
              &\multicolumn{2}{c}{} & \multicolumn{5}{c}{\textbf{Memory Usage in GB for RoBERTa-large}} \\
              &\multicolumn{2}{c}{\textbf{Largest OPT/GPT that can fit}} & \multicolumn{3}{c}{Fixed context length (cl=128)} & \multicolumn{2}{c}{Fixed batch size (bs=64)} \\
            \cmidrule(lr){2-3} \cmidrule(lr){4-6} \cmidrule(lr){7-8}  \textbf{Method} & A100 (40GB)    &  H100 (80GB) & $\textrm{bs}=16$& $\textrm{bs}=32$& $\textrm{bs}=64$&
            $\textrm{cl}=256$&
            $\textrm{cl}=512$\\
            \midrule
            MeZO   & 6.7B  & 13B & 2.07 (69\%) & 2.21 (79\%) & 2.51 (88\%) & 3.35 & 5.97\\
            \hdashline
            MeZO-SVRG  & \textbf{2.7B} & \textbf{6.7B}  & \textbf{4.36 (35\%)} & \textbf{4.51 (58\%)} & \textbf{4.72 (76\%)} & 5.13 & 8.02 \\
             FO-SGD  & 1.6B   & 2.7B  & 6.74 & 10.67 & 18.55 & OOM & OOM\\
             FO-Adam & 350M & 1.3B & 10.44 & 14.33 & 22.41 & OOM & OOM\\
            \bottomrule
        \end{tabular}}
    \caption{Shows the largest AR models that can fit on single 40, 80GB GPUs. We also measure the memory usage under different batch sizes (bs) and context lengths (cl) when fine-tuning RoBERTa-large.  Percentages indicate the memory savings with respect to FO-SGD.}
    \label{tab:memory}
\end{table*}

In this section, we evaluate MeZO-SVRG on a variety of fine-tuning tasks by comparing the performance against MeZO \citep{malladi2023mezo} and memory usage against first-order stochastic gradient descent (FO-SGD) \citep{sgd} and first-order Adam (FO-Adam) \citep{Kingma2015AdamAM}. We demonstrate empirically that MeZO-SVRG performs well in the absence of input prompts: it is able to significantly reduce the performance gap to FO methods and consistently surpasses MeZO's performance on a variety of fine-tuning tasks with significantly lower computation time. Furthermore, MeZO-SVRG necessitates a considerably smaller memory footprint compared to FO-SGD and FO-Adam.

\textbf{Setup.} 
We evaluate on both full (FP32) and half (BF16) precision. We detail the experiment results for the BF16 setting in Appendix \ref{appendix:hprecision}.
We mainly consider a prompt-free fine-tuning setting (more challenging loss landscape) but include prompted results for RoBERTa-large \citep{robertalarge} in Appendix \ref{appendix:robertalarge}. All experiments are run on a single GPU; specifically, we consider Nvidia A100 40GB or H100 80GB GPUs. We evaluate the algorithms under two fine-tuning strategies: full- and partial-parameter fine-tuning. In the latter we fine-tune the last layers of the chosen models. We define a query as one forward pass for a single sample. For a fair comparison between MeZO and MeZO-SVRG, we ensured that the total number of queries used by both remains the same; thus, as MeZO-SVRG accrues more queries per step due to the fullbatch gradient estimates, MeZO was run for more steps. Further details of the experiment setup and implementation are provided in Appendices \ref{appendix:experimentsetup} and \ref{appendix:ablations}.

\textbf{Dataset.} We fine-tune on tasks from the NLP GLUE and SuperGLUE benchmarks: Multi-Genre Natural Language Inference Corpus (MNLI), Stanford Question Answering Dataset (QNLI), Stanford Sentiment Treebank (SST-2), Corpus of Linguistic Acceptability (CoLA), and BoolQ \citep{williams2018broadmnli, wang-etal-2018-glue, socher-etal-2013-recursivesst2, warstadt2018neuralcola, superglue}. Similar to \citet{malladi2023mezo}, for each task, our experiments are conducted in a many-shot fine-tuning setting: 512 training examples, 256 validation examples and 256 test samples are randomly sampled from the dataset. 


\textbf{Language Models.} We considered Distilbert \citep{sanh2020distilbert} and RoBERTa-large as our masked LMs. Details on the hyperparameter configuration used for these experiments are provided in Appendix \ref{appendix:distilbert}, \ref{appendix:robertalarge} and \ref{appendix:hprecision}. 
We extend our evaluation to fine-tuning larger autoregressive (AR) models. We consider the GPT2 \citep{radford2019languagegpt2}, OPT-2.7B, and OPT-6.7B \citep{zhang2022opt} models. The hyperparameter configurations used for these experiments are detailed in Appendix \ref{appendix:autoregressive} and \ref{appendix:hprecision}. 

\subsection{LM Fine-tuning Performance}
\begin{table*}[h!]
\centering{\footnotesize
            \begin{tabular}{lcccccccc}
            \toprule
              & \multicolumn{4}{c}{\textbf{GPT2}} & \multicolumn{4}{c}{\textbf{OPT-2.7B}}  \\
            \cmidrule(lr){2-5} \cmidrule(lr){6-9}
             \textbf{Method} & MNLI & QNLI & SST-2 & CoLA & MNLI & QNLI & SST-2 & CoLA\\
            \midrule
            MeZO   & 0.4  & 5.5 & 19.4 & 2.8 & 2.6 & 5.3 & 48 & 55 \\
            MeZO-SVRG   &  \textbf{0.3}  & \textbf{1.9} & \textbf{5.6} & \textbf{2.2} & \textbf{1.1} & \textbf{2.7} & \textbf{25} & \textbf{1.4}\\
            \bottomrule
            \end{tabular}}
\caption{Required GPU-hrs to achieve equivalent performance levels for MeZO-SVRG and MeZO.}
\label{tab:gpuhrs}
\end{table*}


\textbf{MeZO-SVRG significantly outperforms MeZO in both the fine-tuning loss convergence and test accuracy.} On all models and tasks, MeZO-SVRG improves on the test accuracy over MeZO: we see an improvement of up to 20\% in Tables \ref{tab:accuraciesmasked}, \ref{tab:accuraciesautoregressive} and Figure \ref{fig:accuracyrobertatime}. MeZO-SVRG also consistently achieves an improved fine-tuning loss compared to MeZO. This is particularly evident in Figure \ref{fig:convergencerobertaquery}. Additional results are presented in Appendices \ref{appendix:distilbert}, \ref{appendix:robertalarge} and \ref{appendix:autoregressive}. 

\textbf{MeZO-SVRG works well on both full and partial fine-tuning.} The improvement over MeZO is consistent across both fine-tuning modes. In partial fine-tuning, MeZO-SVRG often achieves comparable performance to FO-SGD (within 5\%) on several tasks (see Table \ref{tab:accuraciesmasked}). 

\textbf{MeZO-SVRG closes the gap to FO-SGD in training convergence and matches the test accuracy.} Tables \ref{tab:accuraciesmasked} and \ref{tab:accuraciesautoregressive} demonstrate how MeZO-SVRG closes the performance gap with FO-SGD compared to MeZO. 

\textbf{MeZO-SVRG's superior performance to MeZO extends to the low (half) precision (BF16) setting.} We summarize the half-precision results in Appendix \ref{appendix:hprecision}.

\subsection{Memory Usage Profiling}
\begin{figure}[h]
    \centering
    \includegraphics[width=0.75\columnwidth]{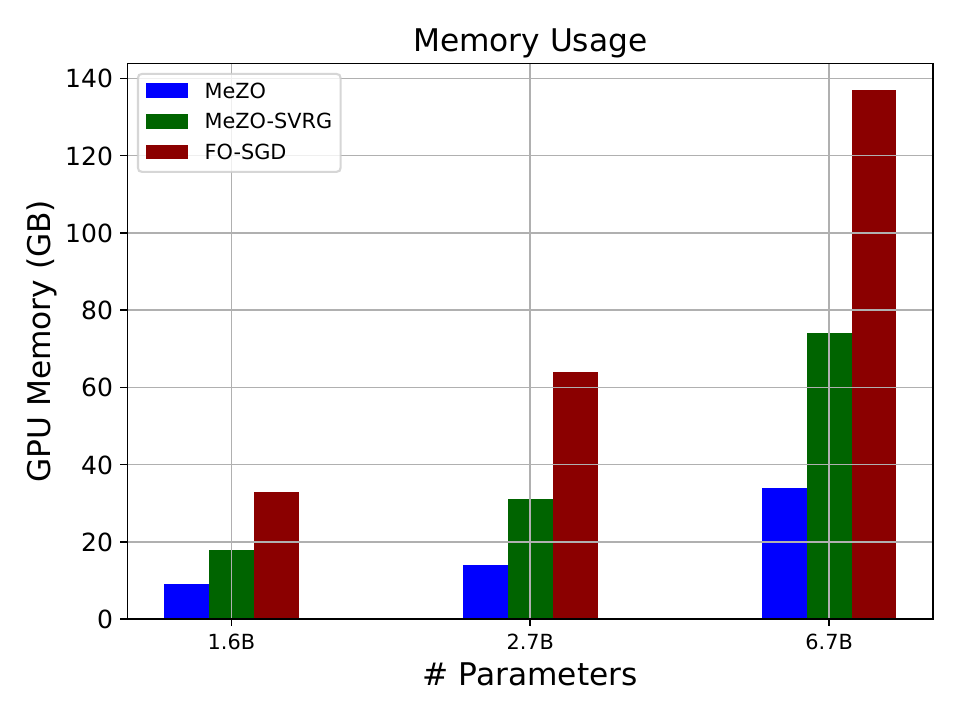}
    \vspace{-1em}
    \caption{Shows the minimum memory usage on autoregressive models ($\textrm{batch size}=1$, use max context length of model). MeZO-SVRG yields a $2\times$ smaller memory footprint over FO-SGD.}
    \label{fig:memoryplot}
\end{figure}
\textbf{MeZO-SVRG can fit larger models on the same hardware than FO-SGD.} We measure the minimum memory requirement to fine-tune (full-parameter) the considered autoregressive models using the different methods. We fine-tune GPT2, OPT-2.7B and OPT-6.7B on MNLI by setting the input sequence length to the maximum context length of the LM and report the peak GPU memory consumption for $\textrm{batch size}=1$. Table \ref{tab:memory} shows that \textit{MeZO-SVRG consistently yields a significantly improved memory footprint compared to FO-SGD (approximately $2\times$ across considered autoregressive models)}. More details on how memory profiling was done is summarized in Appendix \ref{appendix:memory}. 

\textbf{MeZO-SVRG's memory savings progressively improve over FO-SGD and FO-Adam with increasing batch size and context lengths.} For this experiment, we consider the masked model RoBERTa-large. Again we fine-tune on the MNLI dataset using a single Nvidia A100 40GB GPU and set the input sequence length to a constant size of 128. We measure the peak GPU memory consumption for the different methods for varying batch sizes \{16, 32, 64\}. Figure \ref{tab:memory} shows that for a fixed model (RoBERTa-large) and context length (128), MeZO-SVRG exhibits memory savings of up to 76\% w.r.t FO-SGD. We also vary the context lengths \{256, 512\} of the input for a fixed batch size (64). Again we observe significant benefits for MeZO-SVRG over FO-SGD: the latter is subject to out-of-memory errors when running this setting with 40GB GPUs.

MeZO-SVRG consumes more memory than MeZO due to its need to store copies of the parameters and fullbatch SPSA estimators (see Algorithm \ref{alg:method2}), but compensates by delivering notable gains in test performance and computation time. 

\subsection{Computation Time} 
We compare the speed of MeZO-SVRG and MeZO by measuring the total GPU-hours required to achieve a certain performance threshold. For a fair comparison, we set the threshold to a level attained by both methods, namely, MeZO's peak test accuracy. Table \ref{tab:gpuhrs} shows that for GPT2 and OPT-2.7B, MeZO-SVRG consistently achieves superior test accuracy with less than half the GPU-hours.


\subsection{Understanding MeZO-SVRG}
To better understand how the perturbation scale $\mu$ and regularity of full-batch update steps determined by $q$ impact the MeZO-SVRG performance, we perform ablation studies in Appendix \ref{appendix:ablations} with DistilBert on the MNLI dataset. For large fine-tuning datasets, estimating the full-batch gradient can be impractical. Therefore, we included an ablation study to examine the impact on MeZO-SVRG performance when substituting the full-batch gradient estimator with a large-batch estimator. Results in Table \ref{tab:ftlargebatchhp} suggest that large-batch estimators pose an effective alternative to full-batch estimators.

\section{Convergence Theory}\label{section:theory}
In this section, we provide a convergence analysis of MeZO-SVRG. We start by showing that our estimator is unbiased w.r.t. a minibatch set $\mathcal{I}$. We assume $\mathcal{I}$ is drawn either uniformly random with or without replacement. 

\begin{lemma}\label{lemma:unbiasness}
\begin{align}
    \mathbb{E}_\mathcal{I} \bar{\nabla} f_\mathcal{I}(\boldsymbol{\theta})
    = \bar{\nabla} f(\boldsymbol{\theta})    
\end{align}
\end{lemma}

We denote $\textbf{u}_\mathcal{I}= \bar{\nabla} f_\mathcal{I}(\boldsymbol{\theta}) - \bar{\nabla} f_\mathcal{I}(\boldsymbol{\theta}')
- \mathbb{E}_\mathcal{I} [\bar{\nabla} f_\mathcal{I}(\boldsymbol{\theta}) - \bar{\nabla} f_\mathcal{I}(\boldsymbol{\theta}')]
$ and $\textbf{u}_\mathcal{I} = \textbf{u}_i$ for $\mathcal{I} = \{i\}$. 
This $\textbf{u}_i$ is a key component from the idea of control covariates \cite{tucker2017rebar} in reducing variance.

\begin{lemma}\label{lemma:orthogonal}
$\sum_{i=1}^n \mathbf{u}_i = 0$ and $\mathbb{E}_\mathcal{I}[\mathbf{u}_i \mathbf{u}_j] =0$ where $i, j \in \mathcal{I}$ and $i\neq j$.  
\end{lemma}

\textbf{Assumptions. } A1: Functions $\{f_i\}$ are L-smooth, i.e., 
$\|\nabla f_i(\boldsymbol{\theta}) - \nabla f_i(\boldsymbol{\theta}')\| \leq L \|\boldsymbol{\theta}-\boldsymbol{\theta}'\|^2_2$. 
A2: The variance of stochastic gradients is bounded as 
$\frac{1}{n} \sum_{i=1}^n \|f_i(\boldsymbol{\theta}) - f(\boldsymbol{\theta}')\|_2^2 \leq \sigma^2$.

With the two Lemmas and assumptions, the following holds. 
\begin{theorem}\label{theorem:convergence}
Assume A1 and A2 holds. Let learning rates $\eta = \eta_1 = \eta_2$. Then, MeZO-SVRG satisfies
\begin{align}
    \mathbb{E}[\|\nabla f(\boldsymbol{\theta}^{(T)})\|_2^2] 
    \leq 
    \frac{f(\boldsymbol{\theta}^{(0)})-f^*}{T\bar \gamma}
    +
    \frac{L\mu^2}{T\bar \gamma}
    +
    \frac{c}{q\bar \gamma}
    \label{eq:cor_1}
\end{align}    
where $\bar \gamma$ and $c$ are functions of learning rate $\eta$, dimension $d$, minibatch size $b$ and $L, \sigma$. Moreover, 
by setting
\[
\mu = \frac{1}{\sqrt{dT}}, 
\quad \eta = \frac{\rho}{L}, 
\quad q = \lceil \frac{d}{31\rho} \rceil,
\]
where $\rho$ is a universal constant,
MeZO-SVRG satisfies
\begin{align}
    \mathbb{E}[\|\nabla f(\boldsymbol{\theta}^{(T)})\|_2^2] 
    = O\left( 
    \frac{d}{T} + \frac{\mathbf{1}(b<n)}{b}
    \right). \label{eq:cor_2}
\end{align}
\end{theorem}
Theorem \ref{theorem:convergence} demonstrates a linear convergence, inverse proportional to iteration $T$ and $q$. The second term in Eq.~\eqref{eq:cor_1} expresses the effect of $\mu$, the magnitude of perturbation, which is small in practice. In Eq.~\eqref{eq:cor_2}, $q$ is proportional to the problem dimension $d$, which can balance overall computational cost. It also reveals the effect of batch size $b$, indicating larger batch sizes are preferred in terms of iteration counts, which coincide with our empirical observation.  
\vspace{0.1em}
\begin{remark}
The derivation of Theorem~\ref{theorem:convergence} heavily relies on mathematical machinery and flows of original SVRG \cite{johnson2013accelerating} and ZO-SVRG \cite{zosvrg}. However, note the gradient estimators in MeZO-SVRG and ZO-SVRG are different, e.g. different scaling, a single perturbation vector $\textbf{z}$ vs multiple perturbation vectors $\{\textbf{z}_i\}$ against \texttt{RandGradEst} \cite{zosvrg}. This requires careful examination often with different derivation like Lemma~\ref{lemma:unbiasness},~\ref{lemma:orthogonal}, while making sure random vector $\textbf{z}$ is conditioned consistently over sequence of derivations in \cite{zosvrg}. A proof sketch clarifying main distinct steps is provided in Appendix \ref{appendix:theory}. 
\end{remark}

\section{Related work}\label{section:rw}
\textbf{Zeroth-Order Optimization.} Zeroth-order (ZO) methods solve optimization problems without using gradient information. 
This class of methods typically estimates the gradient from function queries. Convergence theory has been developed for ZO stochastic gradient descent (ZO-SGD) in both convex \citep{jamiesonzoconvergence, raginsky,duchi} and non-convex settings \citep{zosvrg, zosvrgcoord, park2020structured}. However, these bounds generally depend on the number of parameters $d$. 
In \citet{malladi2023mezo}, authors demonstrate via fine-tuning experiments that after pre-training and the inclusion of task prompts, the loss landscape is well-behaved enough and can be traversed by ZO-SGD. \citet{zhang2024revisitingzo}  benchmarks the performance of ZO methods in the context of LM fine-tuning. However, despite various advances on variance-reduced techniques ~\cite{johnson2013accelerating, defazio2014saga, park2020linear, lu2021variance} within the FO class, to the best of our knowledge, we are the first to explore the direction of variance-reduced ZO optimization for fine-tuning LMs.  

\textbf{Memory-efficient Backpropagation Strategies.}
LLMs are typically fine-tuned by using FO methods such as SGD \citep{sgd} and Adam \citep{Kingma2015AdamAM}. Several methods have been proposed to handle the memory overheads of backpropagation, for e.g. sparsifying gradients \citep{pmlr-v70-sun17c, Wei2017MinimalEB} and quantizing gradients to lower bit precisions \citep{dettmers20228bit, dettmers2022gptint}. Other techniques to save activation memory during forward and backward pass include Gradient checkpointing \citep{chen2016training} and Flash Attention \citep{dao2022flashattention}. 

\textbf{Gradient-free Adaptation of LLMs.}
The pre-training stage gives LLMs the ability to generalize to tasks for which it has not been explicitly trained. This form of adaptation requires instruction prompts and is referred to as \textit{in-context learning} (ICL). 
While ICL enables quick adaptation of the model to specific tasks, drawbacks of this approach include that current models are constrained to limited context window and are sensitive to both the choice of input prompts and demonstrations \citep{malladi2023mezo}. Moreover, it has been empirically demonstrated that ICL on large models performs worse than full fine-tuning on medium-scale models \citep{llmsarefewshotlearners}. Another paradigm of adapting LLMs without using gradients is by using evolutionary algorithms \citet{sun2022bbt, sun2022bbtv2}, however the effectiveness of these methods has not been verified beyond smaller LMs. 

\section{Conclusion}
This work introduces MeZO-SVRG: a variance-reduced ZO method that addresses the challenge of fine-tuning LMs under memory constraints. MeZO-SVRG is a variant of ZO-SVRG that exploits in-place operations for memory-frugality and efficient gradient estimators that accommodate data parallelism for significant improvement in the iteration speed. The method combines fullbatch and minibatch information to yield low variance gradient estimators. We demonstrate empirically that MeZO-SVRG outperforms MeZO consistently on a variety of LM fine-tuning tasks, even in a challenging non-prompted setting, and requires significantly less GPU-hours to achieve this performance. Furthermore, we show that across model types and fine-tuning tasks, MeZO-SVRG is able to considerably close the performance gap to first-order methods while benefiting from a $2\times$ reduction in memory utilization, which progressively improves with larger batch sizes and context lengths.

We are excited to further explore the potential of MeZO-SVRG. In particular, we aim to examine MeZO-SVRG's performance when coupled with PEFT (LoRA, prefix-tuning) and settings where gradient-information is unavailable, e.g. prompt-tuning black-box models that are accessible only through an API. Finally, our work paves the way for exploring a broader spectrum of variance reduction techniques for ZO methods in the context of LM fine-tuning.


\section*{Impact}\label{ssection:impact}
One of the main challenges associated with adapting foundation models to specialized domains is the prohibitive computational burden during the fine-tuning phase. The high computational cost restricts the widespread use of larger models in resource-constrained settings. This deters a wider consumer-base from reaping the benefits of large models and, in turn, limits the democratization of the technology. Moreover, such constraints can be detrimental to the research community as the large-scale computational resources required for model adaptation are available to only a small fraction of researchers and users.
This challenge can be overcome by investigating optimization methods that leverage the memory frugality of inference passes to develop effective fine-tuning strategies.




\bibliography{example_paper}
\bibliographystyle{icml2024}

\newpage
\appendix
\onecolumn
\section{Proof of Theorem}\label{appendix:theory}
Throughout the proof, we drop bold notations $\boldsymbol{\theta}, \mathbf{z}, \mathbf{u} \rightarrow \theta, z, u$ for notational simplicity.

\begin{lemma}\label{lemma:unbiasness_appendix}
\begin{align}
    \mathbb{E}_\mathcal{I} \bar{\nabla} f_\mathcal{I}(\theta)
    = \bar{\nabla} f(\theta)    
\end{align}
\end{lemma}
\begin{proof}
\begin{align*}
    \mathbb{E}_\mathcal{I} \bar{\nabla} f_\mathcal{I}(\theta)
    &= 
    \frac{1}{b} \mathbb{E}_\mathcal{I} \sum_{i\in \mathcal{I}} \frac{f_i(\theta + \mu z) - f_i(\theta - \mu z) }{2\mu}z
    \\
    &=
    \frac{1}{b} \frac{b}{n} 
    \sum_{i=1}^n \frac{f_i(\theta + \mu z) - f_i(\theta - \mu z) }{2\mu}z
    \\
    &=  \bar{\nabla} f(\theta)
\end{align*}
The first and third equality comes from the definition of $\bar\nabla f_{\mathcal{I}}, \bar \nabla f$ and the second equality holds due to re-ordering under the assumption a minibatch set is sampled uniformly random or random with permutation. 
\end{proof}

We denote ${u}_\mathcal{I}= \bar{\nabla} f_\mathcal{I}({\theta}) - \bar{\nabla} f_\mathcal{I}({\theta}')
- \mathbb{E}_\mathcal{I} [\bar{\nabla} f_\mathcal{I}({\theta}) - \bar{\nabla} f_\mathcal{I}({\theta}')]
$ and ${u}_\mathcal{I} = {u}_i$ for $\mathcal{I} = \{i\}$. 

\begin{lemma}\label{lemma:orthogonal_appendix}
$\sum_{i=1}^n {u}_i = 0$ and $\mathbb{E}_\mathcal{I}[{u}_i {u}_j] =0$ where $i, j \in \mathcal{I}$ and $i\neq j$.  
\end{lemma}

\begin{proof}
    By definition, $\sum_{i=1}^n \bar{\nabla} f_i (\theta) = n \bar \nabla f (\theta)$. It is immediate to see  $\sum_{i=1}^n \mathbb{E}_\mathcal{I} [\bar{\nabla} f_\mathcal{I}({\theta})] =  n \bar{\nabla} f({\theta})$, similar to Lemma~\ref{lemma:unbiasness_appendix}. Therefore
    $\sum_{i=1}^n {u}_i = 0$ holds.
    Conditioned on other randomness, e.g. perturbation $z$, 
    $\mathbb{E}_\mathcal{I}[{u}_i{u}_j] =0$ as $i, j$ are independent. 
\end{proof}

\textbf{Assumptions. } A1: Functions $\{f_i\}$ are L-smooth, i.e. 
$\|\nabla f_i(\theta) - \nabla f_i(\theta')\| \leq L \|\theta-\theta'\|^2_2$. 
A2: The variance of stochastic gradients is bounded as 
$\frac{1}{n} \sum_{i=1}^n \|f_i(\theta) - f(\theta')\|_2^2 \leq \sigma^2$.

Equipped with two Lemmas, the following holds 
\begin{theorem}\label{corollary:convergence_appendix}
Assume A1 and A2 holds. Let learning rate $\eta = \eta_1 = \eta_2$. Then, MeZO-SVRG satisfies
\begin{align}
    \mathbb{E}[\|\nabla f(\theta^{(T)})\|_2^2] 
    \leq 
    \frac{f(\theta^{(0)})-f^*}{T\bar \gamma}
    +
    \frac{L\mu^2}{T\bar \gamma}
    +
    \frac{c}{q\bar \gamma}
    \label{eq:cor_1_appendix}
\end{align}    
where $\bar \gamma$ and $c$ are the functions of stepsize $\eta$, dimension $d$, mini-batch size $b$ and $L, \sigma$. Moreover, 
by setting
\[
\mu = \frac{1}{\sqrt{dT}}, 
\quad \eta = \frac{\rho}{L}, 
\quad q = \lceil \frac{d}{31\rho} \rceil
\]
where $\rho$ is a universal constant,
MeZO-SVRG satisfies
\begin{align}
    \mathbb{E}[\|\nabla f(\theta^{(T)})\|_2^2] 
    = O\left( 
    \frac{d}{T} + \frac{\mathbf{1}(b<n)}{b}
    \right). \label{eq:cor_2_appendix}
\end{align}
\end{theorem}

\begin{proof}
    We rely on the proof provided by \citet{zosvrg}. Note that we need to make sure that certain important steps and Lemmas still hold under MeZO-SVRG's gradient estimators. We start by using $d \bar \nabla f$ as our gradient estimate, through which Lemma 1 and 2 (in in \cite{zosvrg}) hold by matching the scale of gradient to \texttt{RandGradEst} in \cite{zosvrg}. Lemma~\ref{lemma:unbiasness_appendix} is used for Eq. (36) (Proposition 1 of \citet{zosvrg}). Lemma~\ref{lemma:orthogonal_appendix} is used for Lemma 4, 5 in \cite{zosvrg}. Eq. (40) (Proposition 1 of \cite{zosvrg}) holds because of a different conditional expectation, i.e., $E = E_z E_{\mathcal{I}|z} = E_z E_{\mathcal{I}} $, rather than $E = E_{\{z_i\}} E_{\mathcal{I}|\{z_i\}} =E_{\{z_i\}} E_{\mathcal{I}} $ where $z$ and $\{z_i\}$ are random perturbations. The rest of proof follows through algebraic inequalities based on Lemmas 1,2, 4,5, and function assumptions, to derive convergence analysis. Finally we scale down learning rate $\eta$ by $d$ to adopt the gradient estimate of our definition.
\end{proof}

\section{Exploring the Limits of MeZO Empirically}\label{appendix:limitsmezo}
\subsection{MNIST classification and RoBERTa-large fine-tuning}
We ran experiments to better understand shortcomings in MeZO \citep{malladi2023mezo}. Two settings were considered: performing MNIST \citep{lecun1998gradient} classification with a two-layer MLP (25K parameters) and fine-tuning RoBERTa-large (350M parameters) on the SST-2 \citep{socher-etal-2013-recursivesst2} dataset. In the former, we used a two-layer feedforward network with 32 and 16 hidden units respectively. In the latter, we performed full-parameter fine-tuning. In \citet{malladi2023mezo}, authors also remark that a simple instruction prompt is needed for the algorithm to succeed in fine-tuning tasks, i.e. it requires a sufficiently well-behaved optimization trajectory. While this, in itself, can be noted as a drawback, we adopted their proposed prompts in the experiment \citep{malladi2023mezo}. The training and fine-tuning runs are illustrated in Figures \ref{fig:mnistbatch} and \ref{fig:robertabatch}. The hyperparameters selected for the runs are summarized in Tables \ref{tab:hyperparameters_batchsize} and \ref{tab:hyperparameters_batchsize2}. We paid particular attention to the effect of varying batch size on the algorithm performance. We also varied the perturbation scale $\mu$ used in the SPSA estimates \eqref{eq:minibatchspsa2}. No improvement was found in reducing $\mu$ from the default setting used in MeZO ($\mu=1e-3$) and thus we present results only for that configuration \citep{malladi2023mezo}. The largest learning rate values used in the grid search were selected for the MeZO runs. As an upper bound reference on performance, we also include the training curves for the FO-SGD algorithm. From both Figures \ref{fig:mnistbatch} and \ref{fig:robertabatch}, it is clear the MeZO has to contend with instability incurred at smaller batch sizes. 

\begin{table*}[ht]
\centering
\caption{The hyperparameter grid optimized over in the initial the small-scale MNIST \citep{lecun1998gradient} classification experiments.}
\label{tab:hyperparameters_batchsize}{\footnotesize
\begin{tabular}{@{}lll@{}}
\toprule
\textbf{Algorithm}    & \textbf{Hyperparameters} & \textbf{Values}        \\ \midrule
MeZO          & Batch size      & $\{32, 64, 128\} \times$ \\
              & Learning rate   & $\{1e-3, 1e-4\} \times$ \\
              & $\mu$   & $\{1e-3, 1e-4, 1e-5\}$ \\\midrule
FO-SGD          & Batch size      & $\{64\} \times$ \\
              & Learning rate  & $\{1e-3\}$ \\
              \bottomrule
\end{tabular}}
\end{table*}

\begin{table*}[ht]
\centering
\caption{The hyperparameter grid optimized over in the initial RoBERTa-large \citep{robertalarge} fine-tuning experiments.}
\label{tab:hyperparameters_batchsize2}{\footnotesize
\begin{tabular}{@{}lll@{}}
\toprule
\textbf{Algorithm}    & \textbf{Hyperparameters} & \textbf{Values}        \\ \midrule
MeZO          & Batch size      & $\{16, 32, 64\} \times$ \\
              & Learning rate   & $\{1e-5, 1e-6\} \times$ \\
              & $\mu$   & $\{1e-3, 1e-4, 1e-5\}$ \\\midrule
FO-SGD          & Batch size      & $\{64\} \times$ \\
              & Learning rate  & $\{1e-5\}$ \\
              \bottomrule
\end{tabular}}
\end{table*}

\subsection{Solving Least Squares}
To make the aforementioned observations even more apparent, we examined the performance of MeZO on a simple linear least-squares (LS) problem. Specifically we solve 
    \begin{align}\label{eq:ls}
        \min_{w\in\R^d} \|Xw-y\|_2^2,
    \end{align}
where $X\in\R^{n\times d}$ is a randomly generated matrix, $w\in\R^d$ is fixed a priori, and $y\in\R^n = Xw +\textrm{noise}$ is the target labels. In our experiment, we focus on the 100-dimensional problem, i.e. with $d=100$ and $n=1000$. For comparison, we also report the performances of our proposed MeZO-SVRG and FO-SGD. The hyperparameter configurations used are presented in Table \ref{tab:hyperparameters_ls}. Figure \ref{fig:leastsquares} makes it clear that MeZO is unable to attain the optimal value and yields a performance gap w.r.t. MeZO-SVRG and FO-SGD.

\begin{table*}[ht]
\caption{The hyperparameters used for the Least Squares (LS) convergence experiment.}
\centering
\label{tab:hyperparameters_ls}{\footnotesize
\begin{tabular}{@{}lll@{}}
\toprule
\textbf{Algorithm}    & \textbf{Hyperparameters} & \textbf{Values}        \\ \midrule
MeZO          & Batch size      & $\{32\} \times$ \\
              & Learning rate   & $\{1e-3\} \times$ \\
              & $\mu$   & $\{1e-3\}$ \\\midrule
MeZO-SVRG          & Batch size      & $\{32\} \times$ \\
              & Learning rate ($\eta_1$)   & $\{1e-3\} \times$ \\
              & Learning rate ($\eta_2$)   & $\{1e-4\}\times$\\
              & $\mu$   & $\{1e-3\}\times$\\
              & $q$ & \{2\}\\\midrule
FO-SGD          & Batch size      & $\{32\} \times$ \\
              & Learning rate  & $\{1e-3\}$ \\
              \bottomrule
\end{tabular}}
\end{table*}

\newpage
\section{Zeroth-Order Stochastic Variance-Reduced Gradient}\label{appendix:zosvrg}
For the sake of completeness, we present the ZO-SVRG algorithm proposed in~\cite{zosvrg}. This algorithm was proposed without a focus on memory efficiency, in contrast to our MeZO-SVRG, which offers significant memory-saving advantages, particularly in the context of fine-tuning large-scale LMs.
\begin{algorithm}[h]
   \caption{ZO-SVRG \cite{zosvrg}
   }
   \label{alg:method1}
   \footnotesize
\begin{algorithmic}
    \STATE {\bfseries Input:} Total iterations $T$, learning rate $\eta>0$, minibatch size $b$, parameters $\boldsymbol{\theta}_0$, iterations between fullbatch estimators $q\in\mathbb{N}$ 
   \STATE {\bfseries begin method}
    \FOR{$t = 0, \dots, T$}
   \IF {$t \mod q = 0$}
   \STATE 1. $\textbf{g} \leftarrow \hat{\nabla}f(\boldsymbol{\theta}^{(t)})$
   \STATE 2. $\bar{\boldsymbol{\theta}} \leftarrow \boldsymbol{\theta}^{(t)}$
   \ENDIF{}
   \STATE 3. Choose mini-batch $\mathcal{I}_t$ of size b
   \STATE 4. $\hat{\textbf{g}}\leftarrow\hat{\nabla}f_{\mathcal{I}_t}(\boldsymbol{\theta}^{(t)})$
   \STATE 5. $\bar{\textbf{g}}\leftarrow\hat{\nabla}f_{\mathcal{I}_t}(\bar{\boldsymbol{\theta}})$
   \STATE 6. Compute gradient blending: $\textbf{v}_t \leftarrow \hat{\textbf{g}} - \bar{\textbf{g}} + \textbf{g}$
   \STATE 7. update: $\boldsymbol{\theta}^{(t+1)} \leftarrow \boldsymbol{\theta}^{(t)} - \eta \textbf{v}^{(t)}$
   \ENDFOR{}
   \STATE {\bfseries end} 
\end{algorithmic}
\end{algorithm}

\newpage
\section{Experiment Setup}\label{appendix:experimentsetup}
\subsection{Datasets}
For experiments on LMs, we considered fine-tuning on classification datasets. Specifically, we focused on the following datasets from the General Language Understanding Evaluation (GLUE) \citep{wang-etal-2018-glue} benchmark: Multi-Genre Natural Language Inference (MNLI) \citep{williams2018broadmnli}, Question Natural Language Inference (QNLI) \citep{wang-etal-2018-glue} for sentence pair classification, Stanford Sentiment Treebank (SST-2) \citep{socher-etal-2013-recursivesst2} for sentiment analysis, and Corpus of Linguistic Acceptability (CoLA) \citep{warstadt2018neuralcola}. To incorporate a more challenging task, we also evaluated on the BoolQ dataset from the SuperGLUE \citep{superglue} benchmark.

The datasets are imported from the Huggingface \texttt{datasets} library. We randomly sampled 512 examples for training, 256 for validation and 256 for testing. 

\subsection{Model}
In our implementation, we used models from the Huggingface \texttt{transformers} package. As we considered classification datasets, we instantiated models from the \texttt{AutoModelsForSequenceClassification} and \texttt{OPTModelsForSequenceClassification} classes. These libraries add a classification head on top of the considered pre-trained model. For the prompted experiment setting, we instantiate from the \texttt{RobertaModelForPromptFinetuning} custom class implemented in the MeZO repository \citep{malladi2023mezo}.

Tables \ref{tab:modelinfo} and \ref{tab:modelinfo2} summarize the models that where considered in our experiments. For the masked models both full- and partial parameter fine-tuning was performed.

\begin{table*}[h!]
\centering
{\scriptsize
\begin{tabular}{l c c c}
\toprule
\textbf{Model} & Total Trainable Parameters ($\times 10^6$)  & Partial Fine-tuning Layers & Partial Fine-tuning Nr. of Parameters ($\times 10^6$) \\
\midrule
DistilBert (\texttt{distilbert-base-cased}) & 66 & $\begin{bmatrix}
    \texttt{transformer.layer.5}\\\texttt{classifier}
\end{bmatrix}$ & 8 \\
\midrule
RoBERTa-large (\texttt{roberta-large}) & 355 & $\begin{bmatrix}
    \texttt{roberta.encoder.layer.20}\\
    \texttt{roberta.encoder.layer.21}\\
    \texttt{roberta.encoder.layer.22}\\\texttt{roberta.encoder.layer.23}\\
    \texttt{classifier}
\end{bmatrix}$ & 38\\
\bottomrule
\end{tabular}}
\caption{An overview of the masked LMs used in the experiments. Both full- and partial-parameter fine-tuning was considered for these LLMs.}
\label{tab:modelinfo}
\end{table*}

\begin{table*}[h!]
\centering
{\scriptsize
\begin{tabular}{l c}
\toprule
\textbf{Model} & Total Trainable Parameters ($\times 10^6$)  \\
\midrule
GPT2 (\texttt{gpt2-xl}) & 1557 \\
\midrule
OPT-2.7B (\texttt{facebook/opt-2.7B}) & 2651 \\
\midrule
OPT-6.7B (\texttt{facebook/opt-6.7B}) & 6658 \\
\bottomrule
\end{tabular}}
\caption{An overview of the autoregressive LMs used in the experiments.}
\label{tab:modelinfo2}
\end{table*}

\newpage
\section{MeZO-SVRG Implementation and Ablations}\label{appendix:ablations}
\subsection{Memory-efficient SPSA}\label{appendix:spsaimplementation}
In our implementation we adopt the memory-efficient strategy of computing the SPSA estimator as proposed in \citet{malladi2023mezo}. Rather than sampling and storing the entire perturbation vector $\textbf{z}\in\R^d$, we sample a random seed and use it to regenerate the random vector when required. This allows in-place perturbations of the optimization parameters which minimizes the memory footprint. The memory-efficient perturbation routine is shown in \ref{alg:memoryeffparameter}. The parameters are perturbed in groups rather than individually, i.e. in Algorithm \ref{alg:memoryeffparameter}, each $\theta_i$ denotes a parameter group (e.g. an entire weight matrix). The scaling factor $s\in \{1, -2\}$ is used to perturb the parameters in a forward and backward direction as required in central difference approximations.

\begin{algorithm}[H]
   \caption{Memory-Efficient Parameter Perturbation}
   \label{alg:memoryeffparameter}
   \footnotesize
\begin{algorithmic}
    \STATE {\bfseries Design choices:} Scaling factor $s\in\{1, -2\}$, perturbation size $\mu$ 
   \STATE {\bfseries Input:} Parameters $\boldsymbol{\theta}$, random seed $r$
   \STATE {\bfseries Return:} Updated parameters $\boldsymbol{\theta}$
   \vspace*{\baselineskip}
   \STATE {\bfseries begin method}
   \STATE 1. Set random seed $r$
    \FOR{$\theta_i\in \boldsymbol{\theta}$}
   \STATE 2. $z_i\sim\mathcal{N}(0,1)$
   \STATE 3. $\theta_i \leftarrow \theta_i + s*z_i*\mu$
   \ENDFOR{}
   \STATE {\bfseries end} 
\end{algorithmic}
\end{algorithm}

In this work, experiments were conducted with single SPSA estimators which require exactly 2 forward passes. In $p$-SPSA, $p$ estimators are computed and averaged. A total of $2p$ forward passes are required to compute a $p$-SPSA estimator. We used the default setting of $p=1$ suggested in \citet{malladi2023mezo} for both MeZO and MeZO-SVRG implementations.

\subsection{Role of the Perturbation Parameter}\label{appendix:perturbation}
We investigated the role of the perturbation parameter $\mu$ in MeZO-SVRG. Recall that $\mu$ defines the forward and backward perturbation scale when computing SPSA estimators \eqref{eq:minibatchspsa2} and \eqref{eq:fullbatchspsa2}. We know from \citet{spsa} that the SPSA estimator is asymptotically unbiased as $\mu\rightarrow 0$. We wanted to see the practical effects of different $\mu$ settings for MeZO-SVRG. Thus we carried out an ablation study where the perturbation parameter was varied. We fine-tune DistilBert \citep{sanh2020distilbert} on the MNLI \citep{williams2018broadmnli} dataset. The experiment settings are summarized in Figure \ref{tab:perturbationscale}.

Figure \ref{fig:perturbationscale} shows how the different values of $\mu$ affected the fine-tuning process of the MeZO-SVRG algorithm. We observe that for a sufficiently small values of $\mu$ (i.e. smaller than $1e-1$) we see no noticeable difference in performance, while larger $\mu$ result in diverging behaviour. Similar findings were also empirically corroborated in \citet{malladi2023mezo}. Thus, throughout our work we used the default value of $\mu=1e-3$.

\begin{figure*}[h]
    \centering
    \begin{minipage}{0.5\columnwidth}
    \centering
        \includegraphics[width=0.8\linewidth]{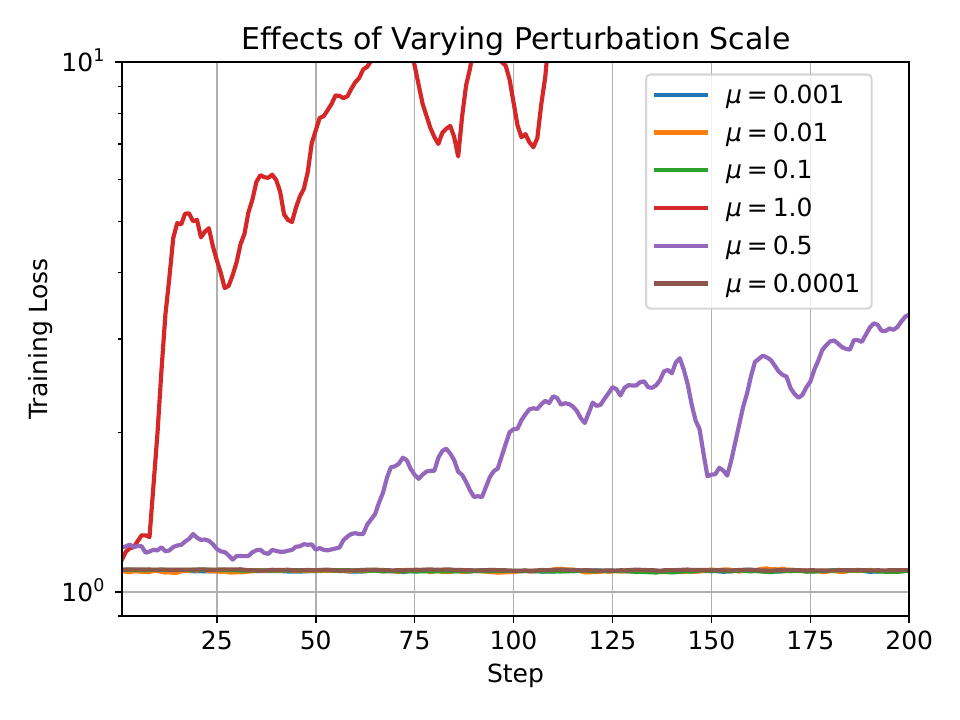}
        \subcaption{}
        \label{fig:perturbationscale}
    \end{minipage}
    \hfill
    \begin{minipage}{0.49\columnwidth}
    \centering
    {\scriptsize
            \begin{tabular}{@{}lll@{}}
            \toprule
            \textbf{Algorithm}    & \textbf{Hyperparameters} & \textbf{Values}        \\ \midrule
            MeZO-SVRG          & Batch size      & $\{64\} \times$ \\
              & Learning rate ($\eta_1$)   & $\{1e-4\} \times$ \\
              & Learning rate ($\eta_2$)   & $\{1e-6\} \times$ \\
              & $\mu$   & $\{1, 0.5, 1e-1, 1e-2, 1e-3, 1e-4\}\times$ \\
              & $q$ & $\{2\}\times$\\
              & Total Steps   & $\{200\}$\\ 
              \bottomrule
            \end{tabular}}
            \subcaption{}
            \label{tab:perturbationscale}
    \end{minipage}
    \caption{a) Shows the effects of varying the perturbation scale on the performance of MeZO-SVRG. b) Shows the hyperparameter settings used in this experiment.}
\end{figure*}

\subsection{Role of $q$}\label{appendix:q}
The parameter $q$ plays a significant role in the performance of MeZO-SVRG (Algorithm \ref{alg:method2}). Concretely, $q$ determines the frequency of fullbatch update steps in the algorithm: smaller $q$ increases the regularity of fullbatch updates. We perform an ablation to better understand the extent to which fullbatch updates help or hinder the MeZO-SVRG performance. We consider the task of fine-tuning the DistilBert \citep{sanh2020distilbert} model on the MNLI \citep{williams2018broadmnli} dataset. The experiment setup is summarized in Figure \ref{tab:varyq}.

Figure \ref{fig:varyq} shows the training curves of MeZO-SVRG for different settings of $q$ over 3500 steps. Increasing the frequency of fullbatch update steps enhances the convergence rate. However, our findings also indicate that a combination of fullbatch and minibatch updates (with $q \geq 2$) contributes to a more stable algorithm performance compared to exclusively using fullbatch updates (when $q=1$).

\begin{figure*}[h]
    \centering
    \begin{minipage}{0.5\columnwidth}
    \centering
        \includegraphics[width=0.8\linewidth]{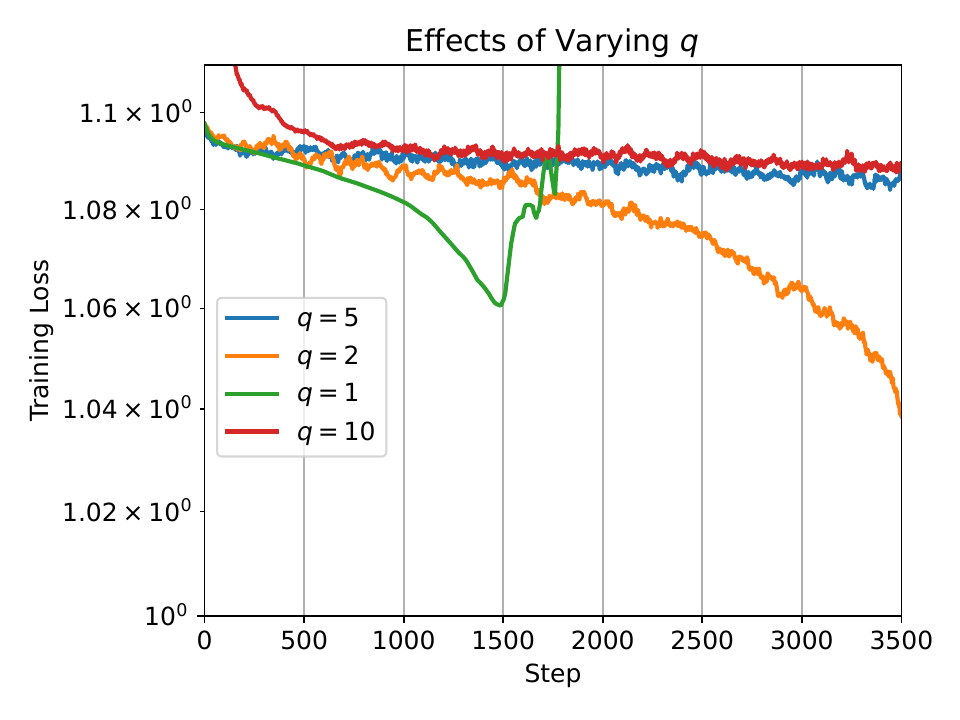}
        \subcaption{}
        \label{fig:varyq}
    \end{minipage}
    \hfill
    \begin{minipage}{0.49\columnwidth}
    \centering
    {\scriptsize
            \begin{tabular}{@{}lll@{}}
            \toprule
            \textbf{Algorithm}    & \textbf{Hyperparameters} & \textbf{Values}        \\ \midrule
            MeZO-SVRG          & Batch size      & $\{64\} \times$ \\
              & Learning rate ($\eta_1$)   & $\{1e-4\} \times$ \\
              & Learning rate ($\eta_2$)   & $\{1e-6\} \times$ \\
              & $\mu$   & $\{1e-3\}\times$ \\
              & $q$ & $\{1, 2, 5, 10\}\times$\\
              & Total Steps   & $\{3500\}$\\ 
              \bottomrule
            \end{tabular}}
            \subcaption{}
            \label{tab:varyq}
    \end{minipage}
    \caption{a) Shows the effects of varying $q$ on the convergence performance MeZO-SVRG. b) Shows the hyperparameter settings used in this experiment.}
\end{figure*}

\subsection{Improved Robustness to Batch Size}
In Figures \ref{fig:leastsquares}, \ref{fig:mnistbatch} and \ref{fig:robertabatch} we emphasize one of the practical drawbacks of MeZO with respect to instability with small batch sizes. We saw this behavior even in the more benign prompted setting. In Figure \ref{fig:bsmezosvrg}, we compare the behavior of MeZO-SVRG and MeZO when fine-tuning RoBERTa-large \citep{robertalarge} on the SST-2 dataset in the prompt-free setting. The plot showcases MeZO-SVRG’s advantage as a low-variance method with improved robustness to different batch sizes. In particular, MeZO's tendencies of diverging with smaller batch sizes are mitigated by MeZO-SVRG. Note that this improvement already becomes apparent over the first 100 iterations of fine-tuning.

\begin{figure}[h]
    \centering
    \includegraphics[width=0.5\columnwidth]{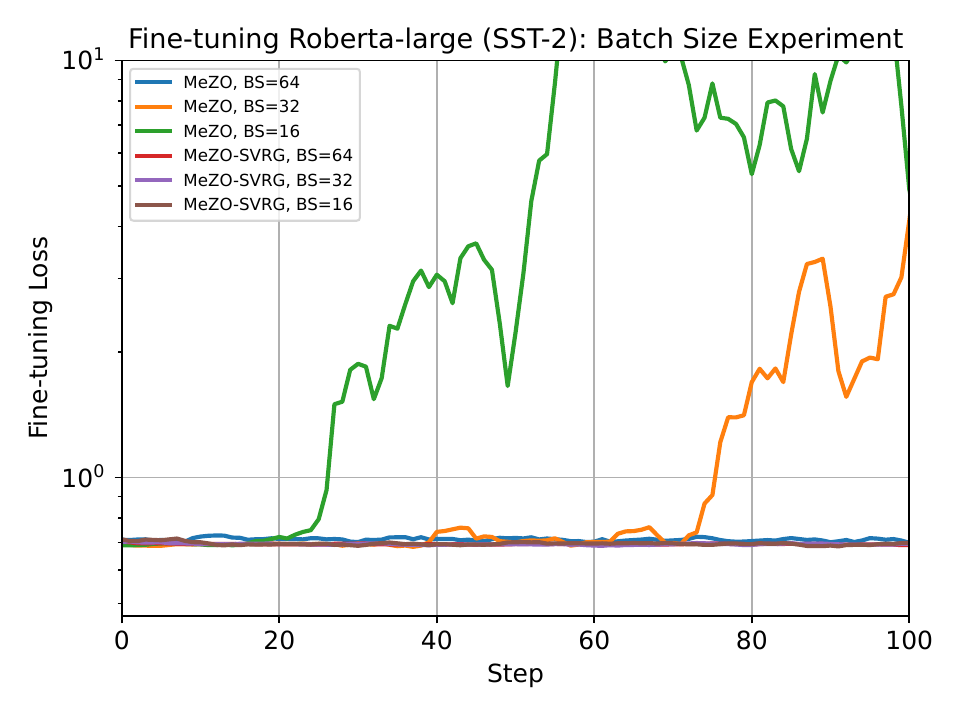}
    \vspace{-1em}
    \caption{Shows improved robustness to smaller batch sizes for MeZO-SVRG compared to MeZO when fine-tuning RoBERTa-large on the SST-2 dataset.}
    \label{fig:bsmezosvrg}
\end{figure}

\subsection{Approximating Fullbatch Estimators with Large Batches}
For sufficiently large training datasets, estimating the fullbatch gradient estimator is prohibitive and time-consuming. Thus we carry out an ablation study to see the effects on the MeZO-SVRG performance when approximating the fullbatch gradient estimator with a large-batch estimator. Specifically, we carry out partial-parameter fine-tuning of DistilBert on a training set of 512 samples for 8000 steps. We choose a mini-batch size of 64 which is consistent across experiment runs. This ablation study is carried out in the half-precision (BF16) setting. We approximate the fullbatch (512 samples) with large batch sizes of 256 and 128. The fine-tuning performances are summarized in Table \ref{tab:ftlargebatchhp}. The obtained results are comparable, suggesting that the large batch-based gradient estimation offers a viable approximation of the fullbatch gradient estimator. 
\begin{table*}[h!]
\centering
\caption{Performance of partial-parameter fine-tuning of DistilBert with half-precision when approximating the fullbatch with large batch sizes. Partial FT refers to partial-parameter fine-tuning (see Appendix \ref{appendix:experimentsetup} for details).}
{\footnotesize
\begin{tabularx}{0.95\linewidth}{l l c c c}
\toprule
\textbf{Task} & \textbf{Method} & \textbf{Fine-tuning Loss} $\downarrow$ & \textbf{Test Accuracy} (\%)$\uparrow$ & \textbf{Queries} ($\times10^3$) $\downarrow$ \\ 
\midrule
SST-2 (Full FT)& MeZO-SVRG (fullbatch$=512$) & 0.4393 & 70 & 2560 \\
 & MeZO-SVRG (Large batch$=256$) & 0.4946 & 71 & 1536 \\
 & MeZO-SVRG (Large batch$=128$)  & 0.5502 & 69 & 1024\\
\bottomrule
\end{tabularx}}

\label{tab:ftlargebatchhp}
\end{table*}

\subsection{Learning Rate Scheduling}
In our implementation, we couple the MeZO-SVRG method with a basic learning rate annealing schedule. This schedule is shown in Algorithm \ref{alg:lrschedule}. This scheduling scheme operates on feedback from training loss values. We compute the average loss values in consecutive epochs. If an increasing trend of average losses is observed, the learning rates are annealed with a factor of $\alpha$. Specifically, if the ratio of leading and trailing average losses is above threshold $\kappa$, we anneal the learning rates. In our experiments we set $\kappa=1.05$ and annealing factor $\alpha=5$.
\begin{algorithm}[H]
   \caption{Learning Rate Scheduling for MeZO-SVRG}
   \label{alg:lrschedule}
   \footnotesize
\begin{algorithmic}
    \STATE {\bfseries Input:} Learning rates $\eta_1, \eta_2$, annealing factor $\alpha$, losses $L$, annealing threshold $\kappa$, total number of batches in an epoch $w$ 
   \STATE {\bfseries begin method}
   \STATE 1. $m_1 \leftarrow \texttt{mean}(L[-w, :])$
   \STATE 2. $m_2 \leftarrow \texttt{mean}(L[-2w, -w])$
   \IF {$\frac{m_1}{m_2}>\kappa$}
   \STATE 3. $\eta_1 \leftarrow \frac{\eta_1}{\alpha}$, $\eta_2 \leftarrow \frac{\eta_2}{\alpha}$
   \ENDIF{}
   \STATE {\bfseries end} 
\end{algorithmic}
\end{algorithm}

\newpage
\section{Fine-tuning DistilBert}\label{appendix:distilbert}
\subsection{Hyperparameter Selection}\label{appendix:distilberthyperparameter}
Table \ref{tab:hyperparametersdistilbert} shows the hyperparameter grid optimized over in the DistilBert \citep{sanh2020distilbert} experiment. The hyperparameter search was done by running the different algorithms for $1$K steps on the MNLI \citep{williams2018broadmnli} dataset and selecting the best configuration. The chosen configuration was then used for a longer fine-tuning runs for all considered tasks, i.e. $200$K steps for MeZO and $50$K steps for MeZO-SVRG. 
\begin{table*}[h]
\centering
\caption{The hyperparameter grid optimized over for the DistilBert \citep{sanh2020distilbert} experiments. In the case of MeZO-SVRG we use the learning rate schedule proposed in Algorithm \ref{alg:lrschedule}. The bold values indicate the configuration used to generate the final results.}
\label{tab:hyperparametersdistilbert}{\footnotesize
\begin{tabular}{@{}lll@{}}
\toprule
\textbf{Algorithm}    & \textbf{Hyperparameters} & \textbf{Values}        \\ \midrule
MeZO          & Batch size      & $\{32, \textbf{64}\} \times$ \\
              & Learning rate   & $\{1e{-4},1e{-5}, \mathbf{1e{-6}}\} \times$ \\
              & $\mu$   & $\{\mathbf{1e{-3}}\}\times$ \\
              & Total Steps   & $\{\mathbf{200K}\}$ \\\midrule
MeZO-SVRG          & Batch size      & $\{32, \mathbf{64}\} \times$ \\
              & Learning rate ($\eta_1$)   & $\{\mathbf{1e{-3}}, 1e{-4}\} \times$ \\
              & Learning rate ($\eta_2$)   & $\{1e{-5}, \mathbf{1e{-6}}\} \times$ \\
              & $\mu$   & $\{\mathbf{1e{-3}}\}\times$ \\
              & $q$ & $\{\mathbf{2}, 5, 10\}\times$\\
              & Total Steps   & $\{\mathbf{50K}\}$ \\\midrule
              FO-SGD          & Batch size      & $\{32, \mathbf{64}\} \times$ \\
              & Learning rate   & $\{1e{-2},\mathbf{1e{-3}}, 1e{-4}\} \times$ \\
              & Total Steps   & $\{\mathbf{1K}\}$ \\
              \bottomrule
\end{tabular}}
\end{table*}

\subsection{Convergence Performance}\label{appendix:convergencedistilbert}
We fine-tune Distilbert \citep{sanh2020distilbert} on the SST-2 \citep{socher-etal-2013-recursivesst2} dataset. In Figure \ref{fig:convergencedistilbertquery}, we show the improved convergence performance of MeZO-SVRG over MeZO. MeZO-SVRG is able to significantly reduce the convergence gap compared to the FO-SGD baseline. Figure \ref{fig:accuracydistilberttime} shows the evolution of the test accuracy over time. Observe that MeZO-SVRG achieves a significant improvement over MeZO in test performance. Moreover, MeZO-SVRG surpasses the peak test accuracy achieved by MeZO in over an order of magnitude less time.

\begin{figure*}[h]
    \centering
    \begin{minipage}{0.49\columnwidth}
    \centering
        \includegraphics[width=0.8\linewidth]{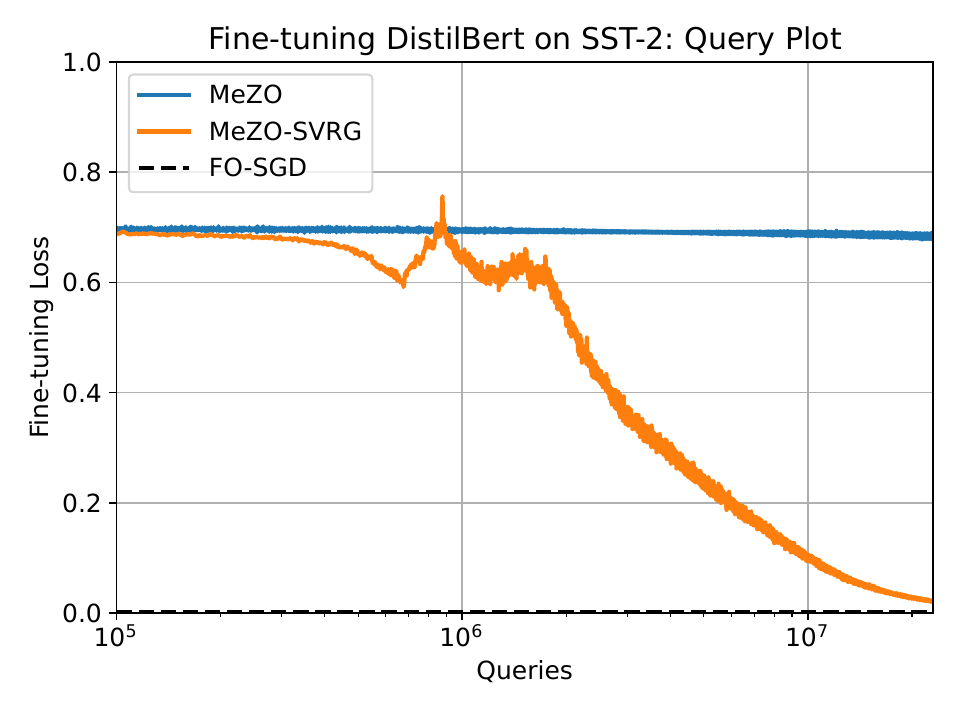}
        \subcaption{}
        \label{fig:convergencedistilbertquery}
    \end{minipage}
    \hfill
    \begin{minipage}{0.49\columnwidth}
    \centering
        \includegraphics[width=0.8\linewidth]{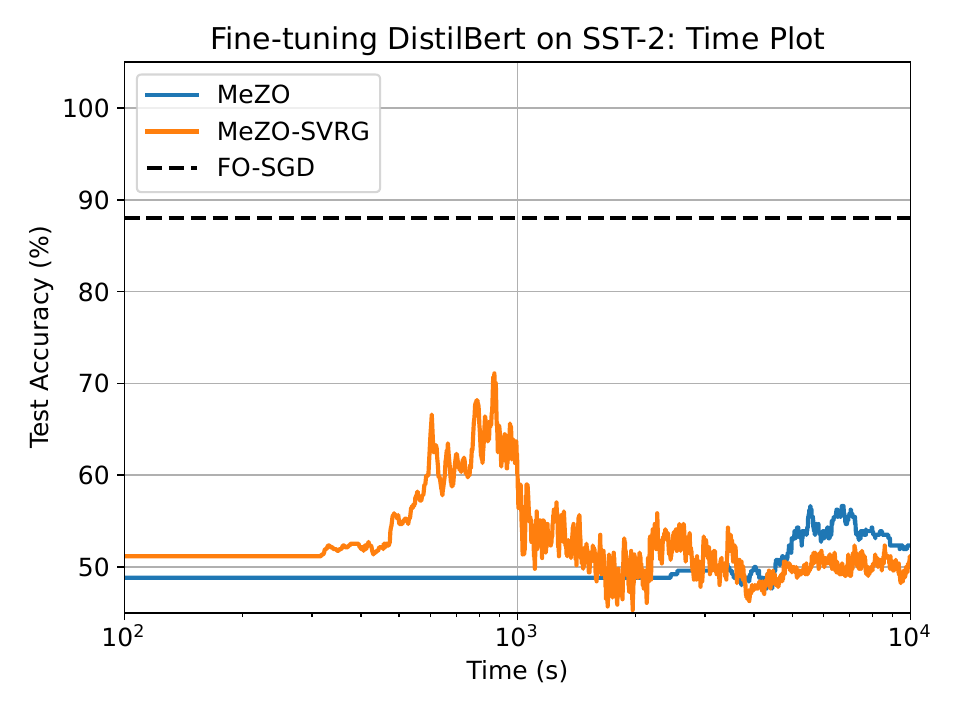}
        \subcaption{}
        \label{fig:accuracydistilberttime}
    \end{minipage}
    \caption{Performance of MeZO-SVRG and MeZO when fine-tuning Distilbert \citep{sanh2020distilbert} on the SST-2 \citep{socher-etal-2013-recursivesst2} dataset. The dashed line serves as a reference to the training loss/test accuracy achieved by FO-SGD. (a) MeZO-SVRG is able to significantly reduce the convergence gap to FO-SGD compared to MeZO. (b) MeZO-SVRG surpasses the peak test performance of MeZO in an order of magnitude less time.}
\end{figure*}

\subsection{Additional Results}\label{appendix:resultsdistilbert}
\begin{table*}[h!]
\centering
\caption{Experiments on DistilBERT (with 512 fine-tuning examples). FO refers to first-order methods. Full FT refers to full-parameter fine-tuning and Partial FT refers to partial-parameter fine-tuning (see Appendix \ref{appendix:experimentsetup} for details).}
{\footnotesize
\begin{tabularx}{0.8\linewidth}{l l c c c}
\toprule
\textbf{Task} & \textbf{Method} & \textbf{Fine-tuning Loss} $\downarrow$ & \textbf{Test Accuracy} (\%)$\uparrow$ & \textbf{Queries} ($\times10^3$) $\downarrow$ \\ 
\midrule
 MNLI (Full FT)& MeZO  & 1.0908 & 36 & 25600\\
 & MeZO-SVRG  & \textbf{0.0757} & \textbf{46} & 25600 \\
  \hdashline
 & FO-SGD  & 0.0101  & 59 & 64 \\
 \midrule
 MNLI (Partial FT)& MeZO & 1.0925 & 35 & 25600 \\
 & MeZO-SVRG & \textbf{0.2775} & \textbf{47} & 25600 \\
 \hdashline
 & FO-SGD  & 0.2617 & 48 & 64 \\
\midrule
 QNLI (Full FT)& MeZO  & 0.6914 & 50 & 25600 \\
 & MeZO-SVRG  & \textbf{0.2335} & \textbf{68} & 25600 \\
 \hdashline
 & FO-SGD  & 0.0372 & 78 & 64 \\
 \midrule
 QNLI (Partial FT)& MeZO & 0.6929 & 52 & 25600 \\
 & MeZO-SVRG  & \textbf{0.2925} & \textbf{65} & 25600 \\
 \hdashline
 & FO-SGD  & 0.4176 & 59 & 64 \\
\midrule
SST-2 (Full FT)& MeZO & 0.6822 & 52 & 25600  \\
 & MeZO-SVRG & \textbf{0.0203} & \textbf{72} & 25600  \\
 \hdashline
 & FO-SGD & 0.0121 & 88 & 64 \\
 \midrule
 SST-2 (Partial FT)& MeZO & 0.6990 & 51 & 25600  \\
  & MeZO-SVRG & \textbf{0.1034} & \textbf{74} & 25600 \\
 \hdashline
 & FO-SGD & 0.0507 & 85 & 64  \\
\midrule
 CoLA (Full FT)& MeZO  & 0.6408 & 62 & 25600 \\
 & MeZO-SVRG  & \textbf{0.2807} & \textbf{68} & 25600 \\

 \hdashline
 & FO-SGD  & 0.0159 & 70 & 64 \\
 \midrule
 CoLA (Partial FT) & MeZO  & 0.6422 & 60 & 25600 \\
  & MeZO-SVRG  & \textbf{0.3617} & \textbf{67} & 25600 \\
 \hdashline
 & FO-SGD  & 0.44719 & 66 & 64 \\
\bottomrule
\end{tabularx}}

\label{tab:distilbertresultslong}
\end{table*}

\newpage
\section{Fine-tuning RoBERTa-large}\label{appendix:robertalarge}

\subsection{Hyperparameter selection}\label{appendix:robertalargehyperparameter}
Table \ref{tab:hyperparametersroberta} presents the hyperparameters searched over in our RoBERTa-large \citep{robertalarge} experiment. The hyperparameter search was done by fine-tuning the model on the MNLI \citep{williams2018broadmnli} dataset for $1$K steps and selecting the best configuration. This selected configuration was subsequently applied to extended fine-tuning sessions across all considered tasks. For our final results, MeZO-SVRG was run for $24$K steps and MeZO was run for $96$K steps.
\begin{table*}[h]
\centering
\caption{The hyperparameter grid optimized over for the RoBERTa-large \citep{robertalarge} experiments. In the case of ZO-SVRG we use the learning rate schedule proposed in Algorithm \ref{alg:lrschedule}. The bold values indicate the configuration used to generate the final results.}
\label{tab:hyperparametersroberta}{\footnotesize
\begin{tabular}{@{}lll@{}}
\toprule
\textbf{Algorithm}    & \textbf{Hyperparameters} & \textbf{Values}        \\ \midrule
MeZO          & Batch size      & $\{32, \mathbf{64}\} \times$ \\
              & Learning rate   & $\{1e{-4},1e{-5}, \mathbf{1e{-6}}\} \times$ \\
              & $\mu$   & $\{\mathbf{1e{-3}}\}\times$ \\
              & Total Steps   & $\{\mathbf{96K}\}$ \\\midrule
MeZO-SVRG          & Batch size      & $\{32, \mathbf{64}\} \times$ \\
              & Learning rate ($\eta_1$)   & $\{1e{-4}, 5e{-5}, \mathbf{1e{-5}}\} \times$ \\
              & Learning rate ($\eta_2$)   & $\{1e{-5}, \mathbf{1e{-6}}\} \times$ \\
              & $\mu$   & $\{\mathbf{1e{-3}}\}\times$ \\
              & $q$ & $\{\mathbf{2}, 5, 10\}\times$\\
              & Total Steps   & $\{\mathbf{24K}\}$ \\\midrule
              FO-SGD          & Batch size      & $\{32, \mathbf{64}\} \times$ \\
              & Learning rate   & $\{1e{-3}, \mathbf{1e{-4}}, 1e{-5}\} \times$ \\
              & Total Steps   & $\{\mathbf{1K}\}$ \\
              \bottomrule
\end{tabular}}
\end{table*}


\subsection{Additional Results}\label{appendix:resultsroberta}
\begin{table*}[h!]
\centering
\caption{Experiments on RoBERTa-large (with 512 fine-tuning examples). Here partial refers to fine-tuning the last layers of the model (see Appendix \ref{appendix:experimentsetup} for details). FO refers to first-order methods. Full FT refers to full-parameter fine-tuning and Partial FT refers to partial-parameter fine-tuning.}
{\footnotesize
\begin{tabularx}{0.8\linewidth}{l l c c c}
\toprule
\textbf{Task} & \textbf{Method} & \textbf{Fine-tuning Loss} $\downarrow$ & \textbf{Test Accuracy }(\%)$\uparrow$ & \textbf{Queries} ($\times10^3$) $\downarrow$ \\ 
\midrule
 MNLI (Full FT) & MeZO  & 0.9447 & 43 & 12288 \\
 & MeZO-SVRG & \textbf{0.8125} & \textbf{49} & 12288\\
  \hdashline
 & FO-SGD  & 0.0292  & 85 & 64 \\
 \midrule
 MNLI (Partial FT)& MeZO  & 1.0729 & 42 & 12288 \\
 & MeZO-SVRG  & \textbf{0.8176} & \textbf{43} & 12288  \\
 \hdashline
 & FO-SGD  & 0.9859 & 52 & 64 \\
\midrule
 QNLI (Full FT) & MeZO  & 0.5845 & 59 & 12288 \\
 & MeZO-SVRG  & \textbf{0.2750} & \textbf{80} & 12288 \\
 \hdashline
 & FO-SGD  & 0.01426 & 89 & 64 \\
 \midrule
  QNLI (Partial FT) & MeZO  & 0.6885 & 50 & 12288 \\
 & MeZO-SVRG  & \textbf{0.4557} & \textbf{67} & 12288 \\
 \hdashline
 & FO-SGD  & 0.5974 & 72 & 64 \\
\midrule
SST-2 (Full FT) & MeZO & 0.69155 & 56 & 12288 \\
 & MeZO-SVRG  & \textbf{0.1336} & \textbf{84} & 12288 \\
 \hdashline
 & FO-SGD  & 0.1086 & 96 & 64 \\
 \midrule
 SST-2 (Partial FT)& MeZO & 0.6837 & 54 & 12288 \\
 & MeZO-SVRG & \textbf{0.5896} & \textbf{72} & 12288  \\
 \hdashline
 & FO-SGD  & 0.5786 & 89 & 64 \\
\midrule
 CoLA (Full FT) & MeZO  & 0.5062 & 68 & 12288 \\
 
 & MeZO-SVRG  & \textbf{0.0644} & \textbf{79} & 12288 \\

 \hdashline
 & FO-SGD  & 0.0099 & 85 & 64 \\
 \midrule
  CoLA (Partial FT) & MeZO & 0.5868 & 65 & 12288 \\
  & MeZO-SVRG  & \textbf{0.3538} & \textbf{79} & 12288 \\
 \hdashline
 & FO-SGD  & 0.4075 & 84 & 64 \\
\bottomrule
\end{tabularx}}

\label{tab:robertaresultslong}
\end{table*}

\begin{table*}[h!]
\centering
\caption{Experiments on RoBERTa-large (with 512 fine-tuning examples) in the prompted setting. Here partial refers to fine-tuning the last layers of the model (see Appendix \ref{appendix:experimentsetup} for details). FO refers to first-order methods. Full FT refers to full-parameter fine-tuning and Partial FT refers to partial-parameter fine-tuning.}
{\footnotesize
\begin{tabularx}{0.95\linewidth}{l l c c c}
\toprule
\textbf{Task} & \textbf{Method} & \textbf{Fine-tuning Loss} $\downarrow$ & \textbf{Test Accuracy }(\%)$\uparrow$ & \textbf{Queries} ($\times10^3$) $\downarrow$ \\ 
\midrule
 MNLI with Prompt (Full FT) & MeZO  &  0.0076 & 73 & 12288 \\
 & MeZO-SVRG & \textbf{0.0058} & \textbf{75} & 12288\\
  \hdashline
 & FO-SGD  & 0.0036 & 96 & 64 \\
 \midrule
 MNLI with Prompt (Partial FT)& MeZO  & 0.4614 & 65 & 12288 \\
 & MeZO-SVRG  & \textbf{0.3177} & 65 & 12288  \\
 \hdashline
 & FO-SGD  & 0.3676 & 81 & 64 \\
\midrule
SST-2 With Prompt (Full FT) & MeZO & \textbf{0.2959} & \textbf{93} & 12288 \\
 & MeZO-SVRG  & 0.3063 & 92 & 12288 \\
 \hdashline
 & FO-SGD  & 0.1578 & 93 & 64 \\
 \midrule
 SST-2 with Prompt (Partial FT)& MeZO & \textbf{0.3280} & 89 & 12288 \\
 & MeZO-SVRG & 0.3393 & 89 & 12288  \\
 \hdashline
 & FO-SGD  & 0.2981 & 90 & 64 \\
\bottomrule
\end{tabularx}}

\label{tab:robertaresultslongprompt}
\end{table*}

\newpage
\section{Additional Results for fine-tuning Autoregressive Models}\label{appendix:autoregressive}
\subsection{Hyperparameter Selection}\label{appendix:hyperparameterautoregressive}
Table \ref{tab:hyperparametersautoregressive} presents the hyperparameter grid searched over for the experiments on autoregressive models. The hyperparameter search was conducted by fine-tuning the models on the MNLI \citep{williams2018broadmnli} dataset for $100$ steps and selecting the best configuration. This selected configuration was used in extended fine-tuning sessions across all considered tasks. For our final results, MeZO-SVRG was run for $8$K steps and MeZO was run for $32$K steps.
\begin{table*}[h]
\centering
\caption{The hyperparameter grid optimized over for the GPT2 \citep{radford2019languagegpt2} and OPT-2.7B \citep{zhang2022opt} experiments. In the case of MeZO-SVRG we use the learning rate schedule proposed in Algorithm \ref{alg:lrschedule}. The bold values indicate the configuration used to generate the final results for both models.}
\label{tab:hyperparametersautoregressive}{\footnotesize
\begin{tabular}{@{}lll@{}}
\toprule
\textbf{Algorithm}    & \textbf{Hyperparameters} & \textbf{Values}        \\ \midrule
MeZO          & Batch size      & $\{32, \mathbf{64}\} \times$ \\
              & Learning rate   & $\{1e{-6}, \mathbf{5e{-6}}, 1e{-7}\} \times$ \\
              & $\mu$   & $\{\mathbf{1e{-3}}\}\times$ \\
              & Total Steps   & $\{\mathbf{32K}\}$ \\\midrule
MeZO-SVRG          & Batch size      & $\{32, \mathbf{64}\} \times$ \\
              & Learning rate ($\eta_1$)   & $\{1e{-4}, \mathbf{5e{-5}}, 1e{-5}\} \times$ \\
              & Learning rate ($\eta_2$)   & $\{\mathbf{1e{-6}}\} \times$ \\
              & $\mu$   & $\{\mathbf{1e{-3}}\}\times$ \\
              & $q$ & $\{\mathbf{2}, 5, 10\}\times$\\
              & Total Steps   & $\{\mathbf{8K}\}$ \\\midrule
              FO-SGD          & Batch size      & $\{8, \mathbf{16}\} \times$ \\
              & Learning rate   & $\{\mathbf{1e{-4}},1e{-5}\} \times$ \\
              & Total Steps   & $\{\mathbf{500}\}$ \\
              \bottomrule
\end{tabular}}
\end{table*}

\subsection{Convergence Performance}\label{appendix:convergencegpt2}
We fine-tune GPT2 \citep{radford2019languagegpt2} and OPT-2.7B \citep{zhang2022opt} on the QNLI \citep{wang-etal-2018-glue} dataset. In Figures \ref{fig:convergencegptquery} and \ref{fig:convergenceoptquery}, we show the improved convergence performance of MeZO-SVRG over MeZO. For both models, MeZO-SVRG is able to significantly reduce the convergence gap compared to the FO-SGD baseline. Figures  \ref{fig:accuracygpttime} and \ref{fig:accuracyopttime} show the evolution of the test accuracy over time. As with the experiments on masked models,  MeZO-SVRG achieves a significant improvement over MeZO in test performance. 

\begin{figure*}[h]
    \centering
    \begin{minipage}{0.49\columnwidth}
    \centering
        \includegraphics[width=0.8\linewidth]{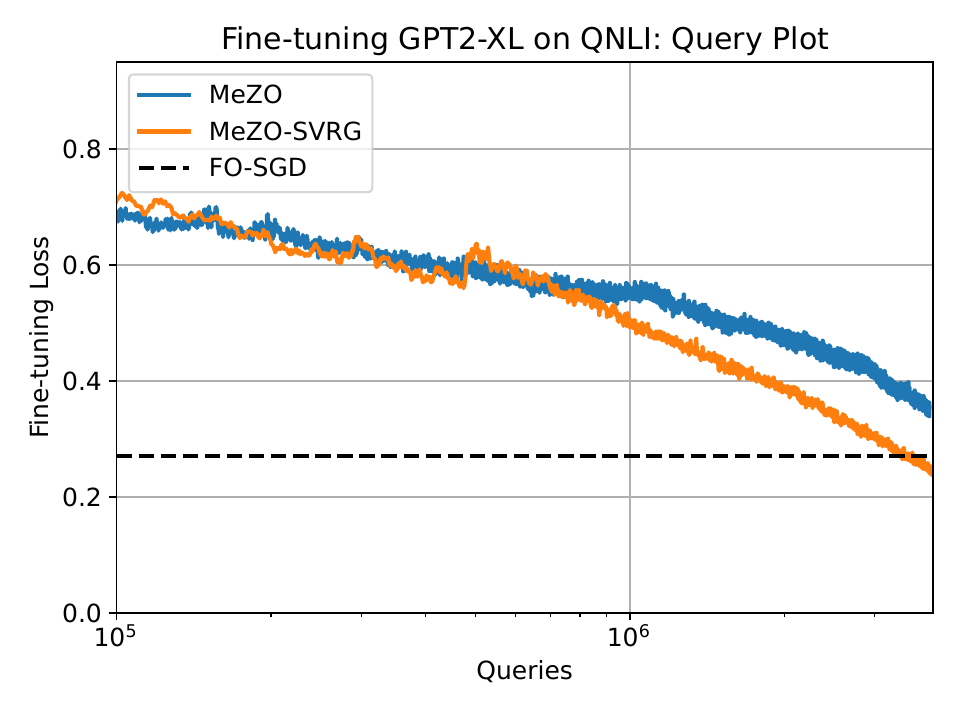}
        \subcaption{}
        \label{fig:convergencegptquery}
    \end{minipage}
    \hfill
    \begin{minipage}{0.49\columnwidth}
    \centering
        \includegraphics[width=0.8\linewidth]{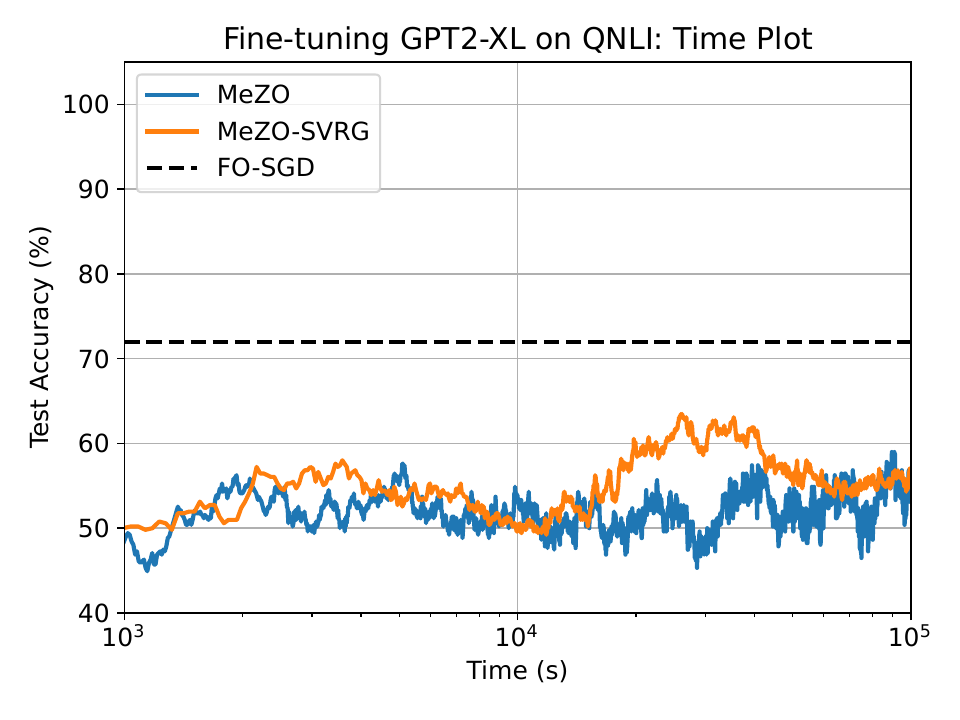}
        \subcaption{}
        \label{fig:accuracygpttime}
    \end{minipage}
    \caption{Convergence performance of MeZO-SVRG, MeZO and FO-SGD when fine-tuning GPT2 \citep{radford2019languagegpt2} on the QNLI \citep{wang-etal-2018-glue} dataset. The dashed line serves as a reference to the training loss achieved by FO-SGD. MeZO-SVRG is able to surpass the fine-tuning loss obtained by FO-SGD. It also improves on the test accuracy attained by MeZO.}
\end{figure*}

\begin{figure*}[h]
    \centering
    \begin{minipage}{0.49\columnwidth}
    \centering
        \includegraphics[width=0.8\linewidth]{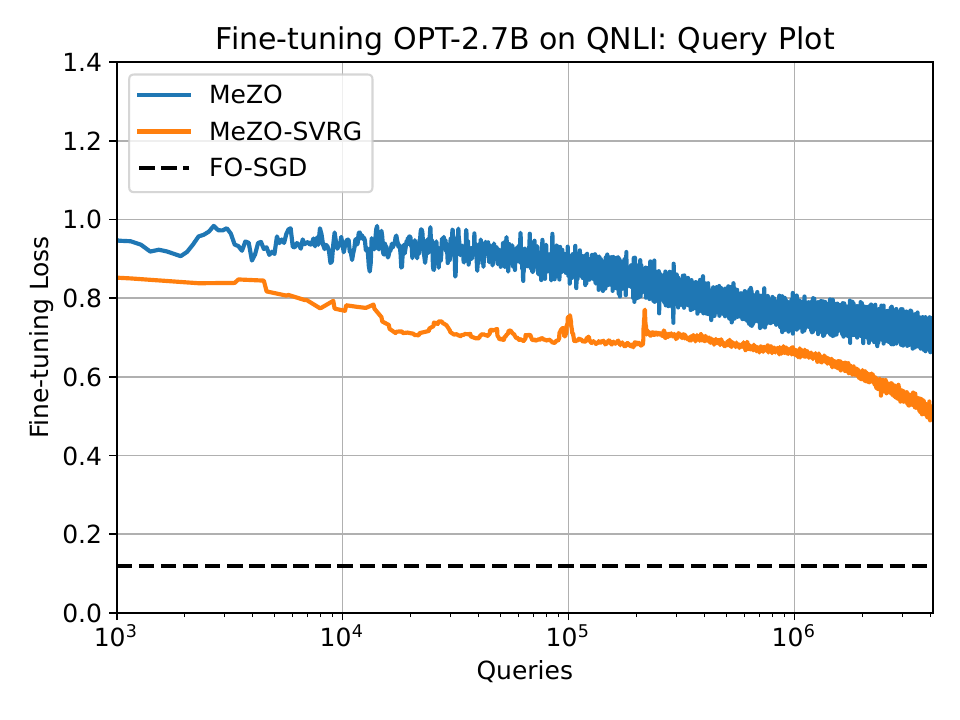}
        \subcaption{}
        \label{fig:convergenceoptquery}
    \end{minipage}
    \hfill
    \begin{minipage}{0.49\columnwidth}
    \centering
        \includegraphics[width=0.8\linewidth]{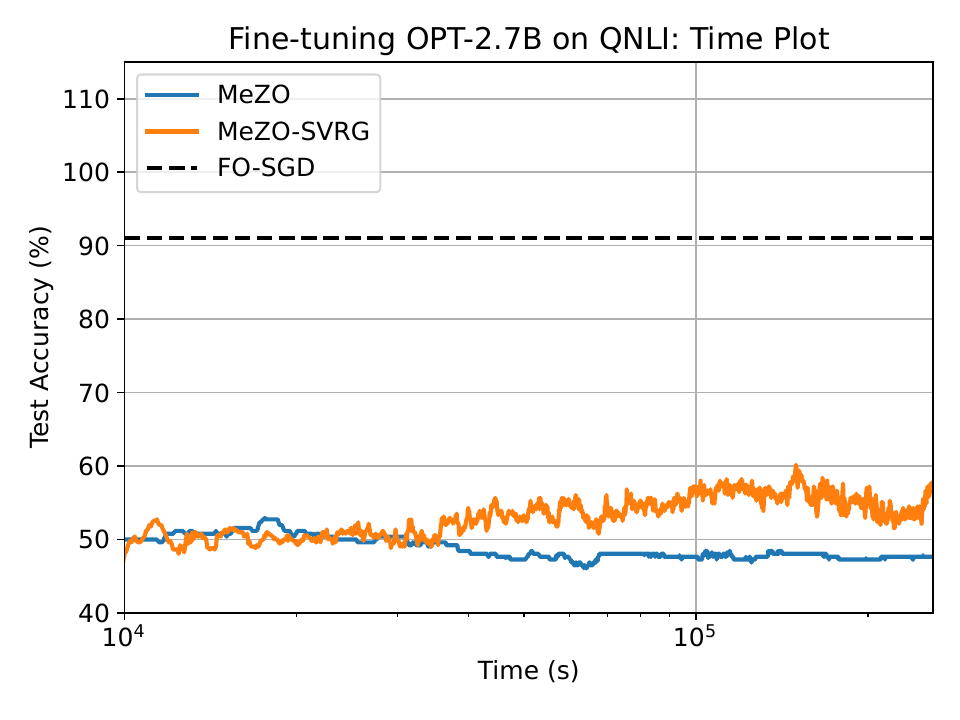}
        \subcaption{}
        \label{fig:accuracyopttime}
    \end{minipage}
    \caption{Performance of MeZO-SVRG, MeZO and FO-SGD when fine-tuning OPT-2.7B \citep{zhang2022opt} on the QNLI \citep{wang-etal-2018-glue} dataset. The dashed line serves as a reference to the training loss/test accuracy achieved by FO-SGD. MeZO-SVRG is able to reduce the convergence gap to FO-SGD compared to MeZO and improve on the test accuracy.}
\end{figure*}

\subsection{Additional Results}\label{appendix:resultsautoregressive}
Tables \ref{tab:gpt2resultslong} and \ref{tab:optresultslong} present extended results on the fine-tuning tasks for GPT2 \citep{radford2019languagegpt2} and OPT-2.7B \citep{zhang2022opt}.
\begin{table*}[h]
\centering
\caption{Experiments on GPT2 (with 512 fine-tuning examples). FO refers to first-order methods. This table summarizes results for full-parameter fine-tuning.}
{\footnotesize
\begin{tabularx}{0.8\linewidth}{l l c c c}
\toprule
\textbf{Task} & \textbf{Method} & \textbf{Fine-tuning Loss} $\downarrow$ & \textbf{Test Accuracy} (\%)$\uparrow$ & \textbf{Queries} ($\times10^3$) $\downarrow$ \\ 
\midrule
 MNLI (Full FT)& MeZO  & 0.6526 & 41 & 4096 \\
 & MeZO-SVRG & \textbf{0.4116} & \textbf{53} & 4096 \\
  \hdashline
 & FO-SGD  & 0.5924  & 69 & 8 \\
\midrule
QNLI (Full FT)
 & MeZO & 0.3351 & 58 & 4096 \\
 & MeZO-SVRG & \textbf{0.2372} & \textbf{63} & 4096 \\
 \hdashline
 & FO-SGD & 0.2799 & 72 & 8 \\
\midrule
 SST-2 (Full FT)& MeZO  & 0.3240 & 59 & 4096 \\
  & MeZO-SVRG  & \textbf{0.2024} & \textbf{65} & 4096 \\
 \hdashline
 & FO-SGD  & 0.2343 & 72 & 8 \\
\midrule
CoLA (Full FT) & MeZO & 0.3544 & 68 & 4096  \\
 & MeZO-SVRG & \textbf{0.2455} & \textbf{69} & 4096 \\
 \hdashline
 & FO-SGD & 0.3855 & 78 & 8 \\
\bottomrule
\end{tabularx}}

\label{tab:gpt2resultslong}
\end{table*}

\begin{table*}[!ht]
\centering
\caption{Experiments on OPT-2.7B (with 512 fine-tuning examples). FO refers to first-order methods. This table summarizes results for full-parameter fine-tuning.}
{\footnotesize
\begin{tabularx}{0.8\linewidth}{l l c c c }
\toprule
\textbf{Task} & \textbf{Method} & \textbf{Fine-tuning Loss} $\downarrow$ & \textbf{Test Accuracy} (\%)$\uparrow$ & \textbf{Queries} ($\times10^3$) $\downarrow$  \\ 
\midrule
 MNLI (Full FT) & MeZO  & 1.0875 & 42 & 4096 \\
 & MeZO-SVRG & \textbf{0.8159} & \textbf{52} & 4096  \\
  \hdashline
 & FO-SGD  &  0.3305 & 78 & 8 \\
\midrule
 QNLI (Full FT) & MeZO  & 0.7026 & 53 & 4096 \\
 & MeZO-SVRG & \textbf{0.4634} & \textbf{60} &  4096 \\
  \hdashline
 & FO-SGD  &  0.1222 & 91 & 8 \\
\midrule
 SST-2 (Full FT) & MeZO  & 0.6530 & 61 & 4096 \\
 & MeZO-SVRG & \textbf{0.5501} & \textbf{65} &  4096 \\
  \hdashline
 & FO-SGD  &  0.0167 & 98 & 8 \\
\midrule
 CoLA (Full FT) & MeZO  & 0.5823 & 62 & 4096 \\
 & MeZO-SVRG & \textbf{0.5335} & \textbf{67} & 4096  \\
  \hdashline
 & FO-SGD  &  0.1724 & 94 & 8 \\
\bottomrule
\end{tabularx}}

\label{tab:optresultslong}
\end{table*}

\newpage
\section{Memory Usage and Computation Time}
\subsection{Memory Profiling}\label{appendix:memory}
We performed memory profiling experiments without any advanced memory-saving options such as lowering bit precision \citep{dettmers20228bit} or gradient check-pointing \citep{chen2016training}. We used full (f32) floating-point precision. 

In the first experiment, we measured the memory requirement needed to run the different methods on full-parameter fine-tuning tasks. The MNLI \citep{williams2018broadmnli} dataset was used to fine-tune autoregressive models GPT2 \citep{radford2019languagegpt2}, OPT-2.7B, OPT-6.7B \citep{zhang2022opt}. We set the input sequence length to the maximum context length for each model, i.e. 1024 for GPT2 and 2048 for the OPT models. The batch size was set to 1. Figure \ref{fig:memoryplot} shows the peak memory consumption in GB as reported by the \texttt{nvidia-smi} command. The peak memory consumption was obtained after executing the methods for at least 100 steps. Table \ref{tab:memory} presents the largest GPT/OPT model that can be fit for each method under the aforementioned settings on single Nvidia A100 40GB and H100 80GB GPUs.

In the second experiment, we measured how the memory usage for the different methods scales with increasing batch size. We fine-tuned RoBERTa-large \citep{robertalarge} on the MNLI \citep{williams2018broadmnli} dataset. The input sequence length was set to a constant 128 and we varied the batch size $\{16, 32, 64\}$. The memory consumption was again measured using the \texttt{nvidia-smi} command and measurements were taken after running the methods for at least 100 steps. Table \ref{tab:memory} summarizes the results. 

We finally also measured how the memory usage varies for the considered algorithms when using a fixed batch size (64) and changing the context length of the input. We used a similar setting to the second experiment: fine-tuning RoBERTa-large \citep{robertalarge} on the MNLI \citep{williams2018broadmnli} dataset. The input context length was varied $\{128, 256, 512\}$ and the memory consumption was measured using the \texttt{nvidia-smi} command. Table \ref{tab:memory} reports the results.

We replicated all experiments in the half-precision (BF16) setting; the results are given in Table \ref{tab:memoryhalf}.

\subsection{Computation Time}\label{appendix:wallclocktime}
We compared the speed of MeZO-SVRG and MeZO \citep{malladi2023mezo} by measuring the time taken by each method to achieve the test performance attained by MeZO. These measurements are based on fine-tuning GPT2 \citep{radford2019languagegpt2} and OPT-2.7B \citep{zhang2022opt} on all considered datasets. Table \ref{tab:gpuhrs} summarizes the results.

\newpage
\section{Half-Precision Experiments}\label{appendix:hprecision}
In the section, we run preliminary experiments to evaluate the considered fine-tuning algorithms on the half-precision (BF16) setting.

\subsection{Half-Precision Experiments on DistilBert}
The hyperparameter grid that was optimized over for the DistilBert experiments in the half-precision setting is presented in Table \ref{tab:hyperparametersdistilberthp}. As each iteration under the half-precision setting is faster than under the full-precision setting, we run experiments for longer. Specifically, we run MeZO-SVRG for 80K steps, MeZO for 400K steps and FO-SGD for 2K steps. The results are summarized in Table \ref{tab:distilbertresultslonghp}.

\begin{table*}[h]
\centering
\caption{The hyperparameter grid optimized over for the half-precision DistilBert \citep{sanh2020distilbert} experiments. In the case of MeZO-SVRG we use the learning rate schedule proposed in Algorithm \ref{alg:lrschedule}. The bold values indicate the configuration used to generate the final results.}
\label{tab:hyperparametersdistilberthp}{\footnotesize
\begin{tabular}{@{}lll@{}}
\toprule
\textbf{Algorithm}    & \textbf{Hyperparameters} & \textbf{Values}        \\ \midrule
MeZO          & Batch size      & $\{32, \textbf{64}\} \times$ \\
              & Learning rate   & $\{1e{-4},\mathbf{1e{-5}}, 1e{-6}\} \times$ \\
              & $\mu$   & $\{\mathbf{1e{-2}}\}\times$ \\
              & Total Steps   & $\{\mathbf{400K}\}$ \\\midrule
MeZO-SVRG          & Batch size      & $\{32, \mathbf{64}\} \times$ \\
              & Learning rate ($\eta_1$)   & $\{\mathbf{1e{-3}}, 1e{-4}\} \times$ \\
              & Learning rate ($\eta_2$)   & $\{\mathbf{1e{-5}}, 1e{-6}\} \times$ \\
              & $\mu$   & $\{\mathbf{1e{-2}}\}\times$ \\
              & $q$ & $\{\mathbf{2}, 5\}\times$\\
              & Total Steps   & $\{\mathbf{80K}\}$ \\\midrule
              FO-SGD          & Batch size      & $\{32, \mathbf{64}\} \times$ \\
              & Learning rate   & $\{\mathbf{1e{-2}}, 1e{-3}, 1e{-4}\} \times$ \\
              & Total Steps   & $\{\mathbf{2K}\}$ \\
              \bottomrule
\end{tabular}}
\end{table*}

\begin{table*}[h!]
\centering
\caption{Half-precision experiments on DistilBERT (with 512 fine-tuning examples). FO refers to first-order methods. Partial FT refers to partial-parameter fine-tuning (see Appendix \ref{appendix:experimentsetup} for details).}
{\footnotesize
\begin{tabularx}{0.8\linewidth}{l l c c c}
\toprule
\textbf{Task} & \textbf{Method} & \textbf{Fine-tuning Loss} $\downarrow$ & \textbf{Test Accuracy} (\%)$\uparrow$ & \textbf{Queries} ($\times10^3$) $\downarrow$ \\ 
\midrule
 MNLI (Partial FT)& MeZO  & 1.0892 & 43 & 51200\\
 & MeZO-SVRG  & \textbf{0.8746} & \textbf{45} & 51200 \\
  \hdashline
 & FO-SGD  & 0.3508 & 51 & 128 \\
 \midrule
 QNLI (Partial FT)& MeZO  & 0.6904 & 60 & 51200 \\
 & MeZO-SVRG  & \textbf{0.5416} & \textbf{64} & 51200 \\
 \hdashline
 & FO-SGD  & 0.2998 & 66 & 128 \\
 \midrule
SST-2 (Partial FT)& MeZO & 0.6889 & 61 & 51200 \\
 & MeZO-SVRG & \textbf{0.3887} & \textbf{79} & 51200 \\
 \hdashline
 & FO-SGD & 0.0555 & 82 & 128\\
\midrule
 CoLA (Partial FT)& MeZO  & 0.6420 & 66 & 51200 \\
 & MeZO-SVRG  & \textbf{0.6170} & \textbf{71} & 51200 \\

 \hdashline
 & FO-SGD  & 0.4218 & 70 & 128\\
\bottomrule
\end{tabularx}}

\label{tab:distilbertresultslonghp}
\end{table*}

\subsection{Half-Precision Experiments on RoBERTa-large}
The hyperparameter grid that was optimized over for the DistilBert experiments in the half-precision setting is presented in Table \ref{tab:hyperparametersrobertalargehp}. As each iteration under the half-precision setting is faster than under the full-precision setting, we run experiments for longer. Specifically, we run MeZO-SVRG for 40K steps, MeZO for 200K steps and FO-SGD for 1K steps. The results are summarized in Table \ref{tab:robertaresultslonghp}.

\begin{table*}[h]
\centering
\caption{The hyperparameter grid optimized over for the half-precision RoBERTa-large \citep{robertalarge} experiments. In the case of MeZO-SVRG we use the learning rate schedule proposed in Algorithm \ref{alg:lrschedule}. The bold values indicate the configuration used to generate the final results.}
\label{tab:hyperparametersrobertalargehp}{\footnotesize
\begin{tabular}{@{}lll@{}}
\toprule
\textbf{Algorithm}    & \textbf{Hyperparameters} & \textbf{Values}        \\ \midrule
MeZO          & Batch size      & $\{\textbf{64}\} \times$ \\
              & Learning rate   & $\{1e{-4},\mathbf{1e{-5}}, 1e{-6}\} \times$ \\
              & $\mu$   & $\{\mathbf{1e{-3}}\}\times$ \\
              & Total Steps   & $\{\mathbf{200K}\}$ \\\midrule
MeZO-SVRG          & Batch size      & $\{\mathbf{64}\} \times$ \\
              & Learning rate ($\eta_1$)   & $\{\mathbf{1e{-4}}, 1e{-5}\} \times$ \\
              & Learning rate ($\eta_2$)   & $\{\mathbf{1e{-5}}, 1e{-6}\} \times$ \\
              & $\mu$   & $\{\mathbf{1e{-3}}\}\times$ \\
              & $q$ & $\{\mathbf{2}, 5\}\times$\\
              & Total Steps   & $\{\mathbf{40K}\}$ \\\midrule
              FO-SGD          & Batch size      & $\{\mathbf{64}\} \times$ \\
              & Learning rate   & $\{\mathbf{1e{-2}}, 1e{-3}, 1e{-4}\} \times$ \\
              & Total Steps   & $\{\mathbf{1K}\}$ \\
              \bottomrule
\end{tabular}}
\end{table*}

\begin{table*}[h!]
\centering
\caption{Half-precision experiments on RoBERTa-large (with 512 fine-tuning examples). FO refers to first-order methods. Partial FT refers to partial-parameter fine-tuning (see Appendix \ref{appendix:experimentsetup} for details).}
{\footnotesize
\begin{tabularx}{0.8\linewidth}{l l c c c}
\toprule
\textbf{Task} & \textbf{Method} & \textbf{Fine-tuning Loss} $\downarrow$ & \textbf{Test Accuracy} (\%)$\uparrow$ & \textbf{Queries} ($\times10^3$) $\downarrow$ \\ 
\midrule
 MNLI (Partial FT)& MeZO  & 1.0898 & 42 & 25600\\
 & MeZO-SVRG  & \textbf{1.0695} & \textbf{43} & 25600 \\
  \hdashline
 & FO-SGD  & 0.1820 & 55 & 64 \\
 \midrule
 QNLI (Partial FT)& MeZO  & 0.6835 & 62 & 25600 \\
 & MeZO-SVRG  & \textbf{0.6070} & \textbf{68} & 25600 \\
 \hdashline
 & FO-SGD  & 0.3112 & 67 & 64 \\
 \midrule
SST-2 (Partial FT)& MeZO & 0.6630 & 66 & 25600 \\
 & MeZO-SVRG & \textbf{0.5278} & \textbf{77} & 25600 \\
 \hdashline
 & FO-SGD & 0.1356 & 93 & 64\\
\midrule
 CoLA (Partial FT)& MeZO  & 0.6308 & 66 & 25600 \\
 & MeZO-SVRG  & \textbf{0.5781} & \textbf{69} & 25600 \\

 \hdashline
 & FO-SGD  & 0.1537 & 88 & 64\\
\bottomrule
\end{tabularx}}

\label{tab:robertaresultslonghp}
\end{table*}

\subsection{Half-Precision Experiments on OPT-6.7B}
The hyperparameter grid optimized for the OPT-6.7B experiments in the half-precision setting is detailed in Table \ref{tab:hyperparametersopt7hp}. We conducted the MeZO-SVRG experiments for 8k steps, MeZO for 24k steps, and FO-SGD for 1k steps. The outcomes of these experiments are summarized in Table \ref{tab:opt7resultslonghp}. We include the BoolQ dataset from the SuperGLUE \citep{superglue} benchmark to evaluate a more challenging fine-tuning task.

\begin{table*}[h]
\centering
\caption{The hyperparameter grid optimized over for the half-precision OPT-6.7B \citep{zhang2022opt} experiments. In the case of MeZO-SVRG we use the learning rate schedule proposed in Algorithm \ref{alg:lrschedule}. The bold values indicate the configuration used to generate the final results.}
\label{tab:hyperparametersopt7hp}{\footnotesize
\begin{tabular}{@{}lll@{}}
\toprule
\textbf{Algorithm}    & \textbf{Hyperparameters} & \textbf{Values}        \\ \midrule
MeZO          & Batch size      & $\{\textbf{128}\} \times$ \\
              & Learning rate   & $\{1e{-5}, \mathbf{1e{-6}}\} \times$ \\
              & $\mu$   & $\{\mathbf{1e{-3}}\}\times$ \\
              & Total Steps   & $\{\mathbf{24K}\}$ \\\midrule
MeZO-SVRG          & Batch size      & $\{\mathbf{128}\} \times$ \\
              & Learning rate ($\eta_1$)   & $\{\mathbf{1e{-4}}, 1e{-5}\} \times$ \\
              & Learning rate ($\eta_2$)   & $\{1e{-5}, \mathbf{1e{-6}}\} \times$ \\
              & $\mu$   & $\{\mathbf{1e{-3}}\}\times$ \\
              & $q$ & $\{\mathbf{2}, 5\}\times$\\
              & Total Steps   & $\{\mathbf{8K}\}$ \\\midrule
              FO-SGD          & Batch size      & $\{\mathbf{64}\} \times$ \\
              & Learning rate   & $\{1e{-3}, \mathbf{1e{-4}}\} \times$ \\
              & Total Steps   & $\{\mathbf{1K}\}$ \\
              \bottomrule
\end{tabular}}
\end{table*}

\begin{table*}[h!]
\centering
\caption{Half-precision experiments on OPT-6.7B (with 512 fine-tuning examples). FO refers to first-order methods. Full FT refers to full-parameter fine-tuning (see Appendix \ref{appendix:experimentsetup} for details).}
{\footnotesize
\begin{tabularx}{0.8\linewidth}{l l c c c}
\toprule
\textbf{Task} & \textbf{Method} & \textbf{Fine-tuning Loss} $\downarrow$ & \textbf{Test Accuracy} (\%)$\uparrow$ & \textbf{Queries} ($\times10^3$) $\downarrow$ \\ 
\midrule
SST-2 (Full FT)& MeZO & 0.5318 & 74 & 6144 \\
 & MeZO-SVRG & \textbf{0.5278} & \textbf{77} & 6144 \\
 \hdashline
 & FO-SGD & 0.103 & 91 & 128\\
 \midrule
BoolQ (Full FT)& MeZO & 0.6259 & 65 & 6144 \\
 & MeZO-SVRG & \textbf{0.5703} & \textbf{69} & 6144 \\
 \hdashline
 & FO-SGD & 0.2872 & 84 & 128\\
\bottomrule
\end{tabularx}}

\label{tab:opt7resultslonghp}
\end{table*}

\newpage
\subsection{Memory Profiling with Half-Precision}
\begin{table*}
    \centering{\footnotesize
        \begin{tabular}{l cccc cc}
            \toprule
              &\multicolumn{1}{c}{} & \multicolumn{5}{c}{\textbf{Memory Usage in GB for RoBERTa-large}} \\
              &\multicolumn{1}{c}{\textbf{Largest OPT/GPT that can fit}} & \multicolumn{3}{c}{Fixed context length (cl=128)} & \multicolumn{2}{c}{Fixed batch size (bs=64)} \\
            \cmidrule(lr){2-2} \cmidrule(lr){3-5} \cmidrule(lr){6-7}  
            \textbf{Method} & A100 (40GB)  & $\textrm{bs}=16$& $\textrm{bs}=32$& $\textrm{bs}=64$&
            $\textrm{cl}=256$&
            $\textrm{cl}=512$\\
            \midrule
            MeZO   & 13B & 1.03 & 1.13 & 1.25 & 1.39 & 2.66\\
            \hdashline
            MeZO-SVRG  & \textbf{6.7B}  & \textbf{2.10 (39\%)} & \textbf{2.11 (66\%)} & \textbf{2.12 (79\%)} & \textbf{2.27 (90\%)} & \textbf{3.66} \\
             FO-SGD  & 2.7B  & 3.42 & 5.81 & 9.83 & 21.87 & OOM\\
             FO-Adam & 1.3B & 5.85 & 8.07 & 12.16 & 24.29 & OOM\\
            \bottomrule
        \end{tabular}}
    \caption{Memory profiling with half-precision. Shows the largest AR models that can fit on single 40 GPUs. We also measure the memory usage under different batch sizes (bs) and context lengths (cl) when fine-tuning RoBERTa-large.  Percentages indicate the memory savings with respect to FO-SGD.}
    \label{tab:memoryhalf}
\end{table*}

\end{document}